\documentclass{article}

\PassOptionsToPackage{numbers, compress}{natbib}

\usepackage[final]{nips_2018}

\usepackage{amsthm}
\usepackage{amsmath}
\usepackage{amssymb}
\usepackage{algorithm}
\usepackage{algorithmic}

\usepackage{microtype}
\usepackage{graphicx}
\usepackage{fullpage,graphicx}
\usepackage{rotating}
\usepackage{subfigure}
\usepackage{booktabs} %
\usepackage{enumitem}
\usepackage{minipage}

\usepackage[utf8]{inputenc} %
\usepackage[T1]{fontenc}    %
\usepackage{hyperref}       %
\usepackage{url}            %
\usepackage{booktabs}       %
\usepackage{amsfonts}       %
\usepackage{nicefrac}       %
\usepackage{microtype}      %
\usepackage{titlesec}

\titlespacing\section{0pt}{5pt plus 2pt minus 1pt}{0pt plus 2pt minus
  2pt}
\titlespacing\subsection{0pt}{5pt plus 2pt minus 1pt}{0pt plus 2pt
  minus 2pt}

\theoremstyle{definition}
\newtheorem{definition}{Definition}[section]

\newtheorem{theorem}{Theorem}
\newtheorem{lemma}{Lemma}

\newtheorem{condition}{Condition}
\newtheorem*{condition*}{Condition}

\usepackage{tikz}
\usetikzlibrary{fit}
\tikzset{%
  highlight/.style={rectangle,rounded corners,fill=red!15,draw,fill opacity=0.2,thick,inner sep=3pt}
}
\newcommand{\tikzmark}[2]{\tikz[overlay,remember picture,baseline=(#1.base)] \node (#1) {$#2$};}
\newcommand{\Highlight}[1][submatrix]{%
    \tikz[overlay,remember picture]{
    \node[highlight,fit=(left.north west) (right.south east)] (#1) {};}
}

\title{{\sc NAIS-Net}: Stable Deep Networks from Non-Autonomous Differential Equations}

\author{
  Marco Ciccone\thanks{The authors equally contributed.} \\
  Politecnico di Milano\\
  NNAISENSE SA\\
  \texttt{marco.ciccone@polimi.it} \\
  \And
  Marco Gallieri\footnotemark[1]~ \thanks{The author derived the mathematical results.}\\
  NNAISENSE SA\\
  \texttt{marco@nnaisense.com} \\
  \AND
  Jonathan Masci\\
  NNAISENSE SA\\
  \texttt{jonathan@nnaisense.com} \\
  \And
  Christian Osendorfer\\
  NNAISENSE SA\\
  \texttt{christian@nnaisense.com} \\
  \And
  Faustino Gomez\\
  NNAISENSE SA\\
  \texttt{tino@nnaisense.com} \\
}

\begin{document}

\maketitle

\begin{abstract}
  This paper introduces \emph{Non-Autonomous Input-Output Stable
  Network} (NAIS-Net), a very deep architecture where each stacked
  processing block is derived from a time-invariant non-autonomous
  dynamical system. Non-autonomy is implemented by skip connections
  from the block input to each of the unrolled processing stages and
  allows stability to be enforced so that blocks can be
  unrolled  adaptively to a {\em pattern-dependent processing depth}.
  NAIS-Net induces \emph{non-trivial, Lipschitz
  input-output maps}, even for an infinite unroll length.
  We prove that the network is globally asymptotically stable so that
  for every initial condition there is exactly one input-dependent
  equilibrium assuming $tanh$ units, and incrementally stable 
  for ReL units.  An efficient implementation that enforces the
  stability under derived conditions for both fully-connected and
  convolutional layers is also presented. Experimental results
  show how NAIS-Net exhibits stability in practice, yielding a
  significant reduction in \emph{generalization gap} compared to ResNets.
\end{abstract}

\section{Introduction} \label{sec:introduction}
Deep neural networks are now the state-of-the-art in a variety of
challenging tasks, ranging from object recognition to natural language
processing and graph analysis
~\cite{krizhevsky2012a,BattenbergCCCGL17,zilly17a,SutskeverVL14,MontiBMRSB17}.
With enough layers, they can, in principle, learn arbitrarily complex
abstract representations through an iterative process~\cite{greff2016}
where each layer transforms the output from the previous layer
non-linearly until the input pattern is embedded in a latent space
where inference can be done efficiently.

Until the advent of Highway~\cite{srivastava2015} and Residual
(ResNet;~\cite{he2015b}) networks, training nets beyond a certain
depth with gradient descent was limited by the vanishing gradient
problem~\cite{hochreiter1991a,bengio1994a}. These very deep networks
(VDNNs) have skip connections that provide shortcuts for the gradient to
flow back through hundreds of layers. Unfortunately, training them
still requires extensive hyper-parameter tuning, and, even if there
were a principled way to determine the optimal number of layers or
\emph{processing depth} for a given task, it still would be fixed for all
patterns.

Recently, several researchers have started to view VDNNs from a
dynamical systems perspective. Haber and Ruthotto~\cite{haber2017} analyzed
the stability of ResNets by framing them as an Euler integration of an
ODE, and~\cite{lu2018beyond} showed how using other numerical
integration methods induces various existing network architectures such
as PolyNet~\cite{zhang2017}, FractalNet~\cite{larsson2016} and
RevNet~\cite{gomez2017}.
A fundamental problem with the dynamical systems underlying
these architectures is that they are
\emph{autonomous}: the input pattern sets the initial condition, only
directly affecting the first processing stage. This means that if the
system converges, there is either exactly one fixpoint or exactly one
limit cycle~\cite{strogatz2014}. Neither case is desirable from a
learning perspective because a dynamical system should have
input-dependent convergence properties so that
representations are useful for learning. One possible approach to
achieve this is to have a \emph{non-autonomous} system where, at each
iteration, the system is forced by an external input.

\begin{figure*}
    \centering
    \includegraphics[width=\textwidth]{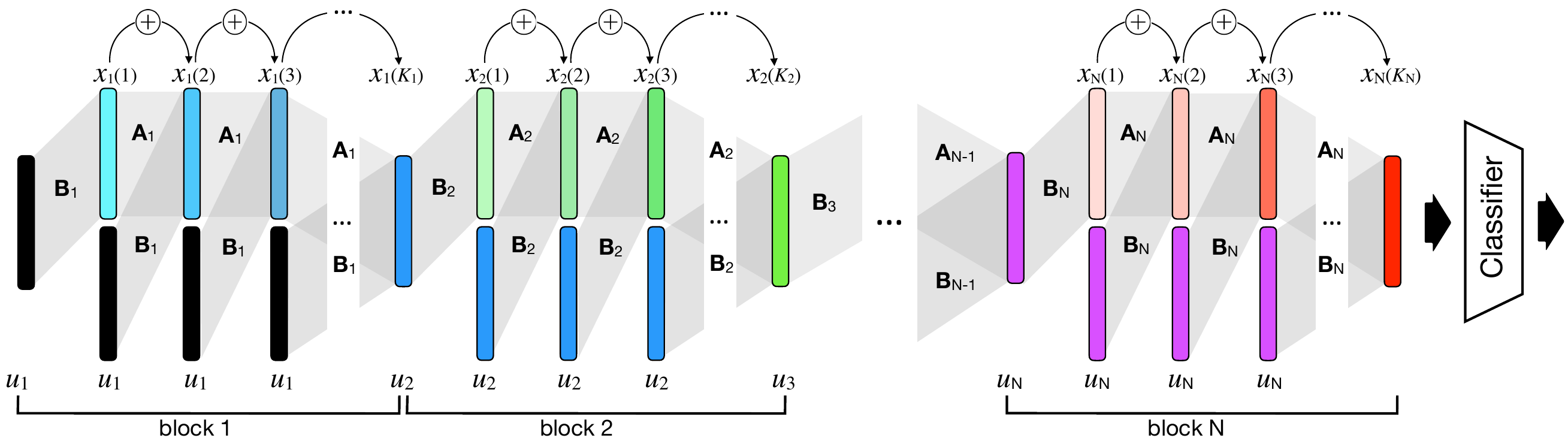}
    \vspace{-0.1cm}
   \caption{\small{\bf NAIS-Net architecture}.  Each block represents a
      time-invariant iterative process as the first layer in the
      $i$-th block, $x_i(1)$, is unrolled into a pattern-dependent
      number, $K_i$, of processing stages, using weight matrices
      $\mathbf{A}_i$ and $\mathbf{B}_i$.  The skip connections from
      the input, $u_i$, to all layers in block $i$ make the process
      non-autonomous. Blocks can be chained together (each block
      modeling a different latent space) by passing final latent
      representation, $x_i(K_i$), of block $i$ as the input to block
      $i+1$.}
    \label{fig:nudenet}
    \vspace{-0.5cm}
\end{figure*}

This paper introduces a novel network architecture, called the
\emph{``Non-Autonomous Input-Output Stable Network''}
(NAIS-Net), that is derived from a dynamical system that is both
time-invariant (weights are shared) and non-autonomous.\footnote{The
  DenseNet architecture~\cite{lang1988,huang2017densely} is
  non-autonomous, but time-varying.}
NAIS-Net is a general residual architecture where a block (see
figure~\ref{fig:nudenet}) is the unrolling of a time-invariant system, and
non-autonomy is implemented by having the external input  applied to each of
the unrolled processing stages in the block through skip connections.
ResNets are similar to NAIS-Net except that ResNets are
time-varying and only receive the external input at the first layer of
the block.

With this design, we can derive sufficient conditions under which the
network exhibits well behaved trajectories for every initial condition. More specifically,
in section~\ref{sec:nonautonomousResNet}, we prove that with $tanh$
activations, NAIS-Net has exactly one input-dependent equilibrium,
while ReLU activations lead to incrementally stable trajectories per
input pattern.  Moreover, the NAIS-Net architecture allows not only the
internal stability of the system to be analyzed but, more importantly,
the input-output stability --- the difference between the representations
generated by two different inputs belonging to a bounded set will also be
bounded at each stage of the unrolling.\footnote{In the supplementary
  material, we also show that these results hold both for shared and
  unshared weights.}

In section~\ref{sec:implementation}, we provide an efficient
implementation that enforces the stability conditions for both
fully-connected and convolutional layers in the stochastic
optimization setting. These implementations are compared
experimentally with ResNets on both CIFAR-10 and CIFAR-100 datasets,
in section~\ref{sec:experiments}, showing that NAIS-Nets achieve
comparable classification accuracy with a much better
\emph{generalization gap}. NAIS-Nets can also be 10 to 20 times deeper
than the original ResNet without increasing the total number of network
parameters, and, by stacking several stable NAIS-Net blocks,
models that implement pattern-dependent processing depth can be
trained without requiring any normalization at each step (except when
there is a change in layer dimensionality, to speed up training).

The next section presents a more formal treatment of the dynamical
systems perspective of neural networks, and a brief overview of work
to date in this area.

\section{Background and Related Work}
\label{sec:background}

Representation learning is about finding a mapping from input patterns
to encodings that disentangle the underlying variational factors of
the input set.  With such an encoding, a large portion of typical
supervised learning tasks (e.g.\ classification and regression) should
be solvable using just a simple model like logistic regression.  A key
characteristic of such a mapping is its invariance to input
transformations that do not alter these factors for a given
input\footnote{Such invariance conditions can be very powerful
  inductive biases on their own: For example, requiring invariance to
  time transformations in the input leads to popular RNN
  architectures~\cite{tallec2018a}.}. In particular, random
perturbations of the input should in general not be drastically
amplified in the encoding. In the field of control theory, this
property is central to stability analysis which investigates the
properties of dynamical systems under which they converge to a single
steady state without exhibiting chaos~\cite{khalil2001,strogatz2014,sontag_book}.

In machine learning,
stability has long been central to the study of
recurrent neural networks (RNNs) with respect to the
vanishing~\cite{hochreiter1991a, bengio1994a, pascanu2013a}, and
exploding~\cite{doya1992bifurcations, baldi1996universal,
  pascanu2013a} gradient problems, leading to the development of Long
Short-Term Memory~\cite{hochreiter1997b} to alleviate the former.
More recently, general conditions for RNN stability have been
presented~\cite{zilly17a, kanai2017a, laurent_recurrent_2016,vorontsov2017} based on
general insights related to Matrix Norm analysis.
Input-output stability~\cite{khalil2001} has also been analyzed for
simple RNNs~\cite{steil1999input,knight_stability_2008,haschke2005input,singh2016stability}.

Recently, the stability of deep feed-forward networks was more closely
investigated, mostly due to adversarial attacks~\cite{szegedy2013intriguing} on
trained networks. It turns out that sensitivity to (adversarial) input
perturbations in the inference process can be avoided by ensuring certain
conditions on the spectral norms of the weight
matrices~\cite{cisse_parseval_2017, yoshida2017a}. Additionally, special
properties of the spectral norm of weight matrices mitigate instabilities
during the training of Generative Adversarial Networks~\cite{miyato2018a}.

Almost all successfully trained
VDNNs~\cite{hochreiter1997b, he2015b, srivastava2015, cho2014learning} share
the following core building block:
\begin{equation}
    \label{eq:resnet_unrolled}
    x(k+1) = x(k) + f\left( x(k), \theta(k) \right), 1 \leq k \leq K.
\end{equation}
That is, in order to compute a vector representation at layer $k+1$
(or time $k+1$ for recurrent networks), \emph{additively update} $x(k)$ with
some non-linear transformation $f(\cdot)$ of $x(k)$ which depends on
parameters $\theta(k)$.
The reason usual given for why Eq.~(\ref{eq:resnet_unrolled}) allows
VDNNs to be trained is that the explicit identity connections
avoid the vanishing gradient problem.

The semantics of the forward path are however still considered
unclear. A recent interpretation is that these feed-forward
architectures implement \emph{iterative inference}~\cite{greff2016,jastrzebski2017}. This view
is reinforced by observing that Eq.~(\ref{eq:resnet_unrolled}) is a
forward Euler discretization~\cite{ascher1998a} of the ordinary
differential equation (ODE) $\dot{{x}}(t) = f(x(t), \Theta)$ if
$\theta(k) \equiv \Theta$ for all $1 \leq k \leq K$ in
Eq.~(\ref{eq:resnet_unrolled}).
This connection between dynamical systems and
feed-forward architectures was recently also observed by several
other authors~\cite{weinan2017a}. This point of view leads to
a large family of new network architectures that are induced
by various numerical integration
methods~\cite{lu2018beyond}. Moreover, stability problems in
both the forward as well the backward path of VDNNs have been addressed
 by relying on well-known analytical approaches for
continuous-time ODEs~\cite{haber2017,chang2017multi}. In the present paper,
we instead address the problem directly in discrete-time, meaning that our
stability result is preserved by the network implementation.
 With the exception of ~\cite{liao_bridging_2016}, none of
this prior research considers time-invariant,
 non-autonomous
systems.

Conceptually, our work shares similarities with approaches
that build network
according to iterative
algorithms~\cite{gregor2010a,zheng2015a} and recent ideas investigating
pattern-dependent processing time~\cite{graves2016a,veit2017a,figurnov2017a}.

\section{Non-Autonomous Input-Output Stable Nets (NAIS-Nets)}
\label{sec:nonautonomousResNet}

This section provides stability conditions for both fully-connected
and convolutional NAIS-Net layers. We formally prove that NAIS-Net
provides a non-trivial input-dependent output for each iteration $k$ as
well as in the asymptotic case ($k \rightarrow \infty$). The following dynamical system:
\begin{equation}
    \label{ResNet_ode4}
    \begin{aligned}
    x({k+1})  = x({k}) + h f\left(x({k}), u, \theta \right), \ %
    x(0)  =0,
    \end{aligned}
\end{equation}
is used throughout the paper, where $x\in\mathbb{R}^{n}$ is the latent state, $u\in\mathbb{R}^{m}$
is the network input, and $h>0$.  For ease of notation, in the remainder of the
paper the explicit dependence on the parameters,
$\theta$, will be omitted.%

\paragraph{Fully Connected NAIS-Net Layer.}
Our fully connected layer is defined by
\begin{equation}
    \begin{aligned}
        x(k+1) %
        &= x(k) + h\sigma\bigg(Ax(k) + Bu + b\bigg),
    \end{aligned}\label{eq:DNN}
\end{equation}
where $A\in\mathbb{R}^{n\times n}$ and $B\in\mathbb{R}^{n\times m}$
are the state and input transfer matrices, and $b\in\mathbb{R}^{n}$ is
a bias.
The activation $\sigma\in\mathbb{R}^n$ is a
vector of (element-wise) instances of an activation function, denoted
as $\sigma_i$ with $i\in\{1,\dots,n\}$.  In this paper, we only consider the
hyperbolic tangent, $tanh$,  and Rectified Linear Units (ReLU) activation
functions.  Note that by setting $B=0$, and the step $h=1$ the original
ResNet formulation is obtained.

\paragraph{Convolutional NAIS-Net Layer.}
The architecture can be easily extended to Convolutional
Networks by replacing the matrix multiplications in Eq. (\ref{eq:DNN})
with a convolution operator:
\begin{equation}
    \label{eq:CNNmodel}
    \begin{aligned}
        X(k+1) %
              & = X(k) + h\sigma\bigg(C * X + D * U + E\bigg).
    \end{aligned}
\end{equation}
Consider the case of $N_C$ channels. The convolutional layer in Eq.
(\ref{eq:CNNmodel}) can be rewritten,  for each latent map
$c\in\{1,2,\dots,N_C\}$, in the equivalent form:
\begin{equation} \label{eq:CNNmodel_rewritten2}
   \begin{aligned}
       X^{c}(k+1) =X^{c}(k)+h\sigma\left(\sum_i ^{N_C}C^{c}_i*X^i(k)
         + \sum_j^{N_C} D^{c}_j*U^j+E^{c}\right),
   \end{aligned}
\end{equation}
where:
$X^i(k)\in\mathbb{R}^{n_X \times n_X }$ is the layer state matrix for
channel $i$, $U^j\in\mathbb{R}^{n_U \times n_U}$ is the layer input
data matrix for channel $j$ (where an appropriate zero padding has
been applied) at layer $k$, $C^{c}_i\in\mathbb{R}^{n_C \times n_C}$ is
the state convolution filter from state channel $i$ to state channel
$c$, $D^{c}_j$ is its equivalent for the input, and $E^{c}$ is a bias.
The activation, $\sigma$, is still applied element-wise.  The
convolution for $X$ has a fixed stride $s=1$, a filter size $n_C$ and
a zero padding of $p\in\mathbb{N}$, such that $n_C=2p+1$.\footnote{ If
  $s\geq0$, then $x$ can be extended with an appropriate number of
  constant zeros (not connected). }

Convolutional layers can be rewritten in the same form as fully
connected layers (see proof of Lemma \ref{lem:CNNconn} in the supplementary material).
Therefore, the stability results in
the next section will be formulated for the fully connected case, but apply
to both.

\paragraph{Stability Analysis.}
Here, the stability conditions for NAIS-Nets which were instrumental to their
design are laid out. We are interested in using a cascade of unrolled NAIS
blocks (see Figure~\ref{fig:nudenet}), where each block is described by either
Eq.~(\ref{eq:DNN}) or Eq.~(\ref{eq:CNNmodel}). Since we are dealing with a
cascade of dynamical systems, then stability of the entire network can be
enforced by having stable blocks~\cite{khalil2001}.

The state-transfer Jacobian for layer $k$ is defined as:
\begin{equation}\label{eq:jac_gen}
	J(x(k),u) = \frac{\partial x(k+1)}{\partial x(k)}=I+h\frac{\partial\sigma(\Delta x(k))}{\partial \Delta x(k)}A,
\end{equation}
where the argument of the activation function, $\sigma$, is denoted
as $\Delta x(k)$.
Take an arbitrarily small scalar $\underline{\sigma}>0$ and define the set of pairs $(x,u)$ for which the activations are not saturated as:
\begin{equation}
    \mathcal{P}=\left\{(x,u):\frac{\partial\sigma_{i}(\Delta x(k))}{\partial\Delta x_i(k)}\geq \underline{\sigma},\ \forall i\in[1,2,\dots,n]\right\}.
\end{equation}

Theorem~\ref{th:BIBO1} below proves that the non-autonomuous residual
network produces a bounded output given a bounded, possibly noisy, input, and
that the network state  converges to a constant value as the number of layers
tends to infinity, if the following stability condition holds:

\begin{condition}\label{assum:simple_jac}
For any $\underline{\sigma}>0$, the Jacobian
satisfies:
\begin{equation} \label{eq:2norm}
		\bar{\rho}=\sup_{(x,u)\in\mathcal{P}}\rho(J(x,u)),\ \text{s.t.}\ \bar{\rho}<1,
\end{equation}
    where $\rho(\cdot)$ is the spectral radius.
\end{condition}

For $\tanh$ activation, the steady states, $\bar{x}$, are determined by a \emph{continuous} function
of $u$. This means that a small change in $u$
cannot result in a very different $\bar{x}$. In particular, $\bar{x}$
depends linearly on $u$, therefore the block needs to be unrolled for a
 finite number of iterations, $K$, for the mapping to be non-linear.
  That is not the case for ReLU, which can
 be unrolled indefinitely and still provide a piece-wise affine mapping which is locally Lipschitz. 

In Theorem \ref{th:BIBO1}, the Input-Output (IO) gain function, $\gamma(\cdot)$, describes the effect
of norm-bounded input perturbations on the network trajectory. This
gain provides insight as to the level of robust invariance of the
classification regions to changes in the input data with respect to
the training set. In particular, as the gain is decreased, the
perturbed solution will be \emph{closer} to the solution obtained from
the training set.  This can lead to increased robustness and
generalization with respect to a network that does not statisfy
Condition~\ref{assum:simple_jac}.
Note that the IO gain, $\gamma(\cdot)$, is linear, and hence the block IO
map is \emph{Lipschitz} even for an \emph{infinite} unroll length. The IO gain
 depends directly on the norm of the
state transfer Jacobian, in Eq. (\ref{eq:2norm}), as indicated by the
term $\bar{\rho}$ in Theorem \ref{th:BIBO1}.\footnote{see supplementary material for additional details and
all proofs, where the untied case is also
covered.}

\setcounter{theorem}{0}%
\begin{theorem} (Asymptotic stability for shared weights)\\
If Condition~\ref{assum:simple_jac} holds, then NAIS-Net with
 $tanh$ activation is Asymptotically Stable with respect to
\emph{input dependent} equilibrium points. More formally:
\begin{equation}
    x(k)\rightarrow\bar{x} \in\mathbb{R}^n,\
    \forall x(0)\in\mathcal{X}\subseteq\mathbb{R}^n,\ u\in\mathbb{R}^m.
\end{equation}
The trajectory is described by $\|x(k)-\bar{x}\|\leq\bar{\rho}^k \|x(0)-\bar{x}\|$
, where $\|\cdot\|$ is a suitable matrix norm.%

\newpage
In particular:
\setlist[itemize]{leftmargin=*}
\begin{itemize}
\item The steady state $\bar{x}$ is independent of the initial
state, and it is a linear function of the input, namely, $\bar{x}=A^{-1}Bu$.
    The network is Globally Asymptotically Stable.

\item The network is Globally
    Input-Output (robustly) Stable for any additive input perturbation $w\in\mathbb{R}^m$.
    The trajectory is described by:
    \begin{equation}\label{eq:ISSgain_tied}
	      \|x(k)-\bar{x}\|\leq\bar{\rho}^{k} \|x(0)-\bar{x}\|+ \gamma(\|w\|),\ \text{with}\   \gamma(\|w\|)=h\frac{\|{B}\|}{(1-\bar{\rho})} \|w\|.
    \end{equation}
    where $\gamma(\cdot)$ is the input-output gain.
For any $\mu\geq0$, if $\|w\|\leq\mu$
then  the following set is robustly positively invariant ($x(k)\in\mathcal{X}, \forall k\geq0$):
    \begin{equation}
       {\mathcal{X}}=
       \left\{
           x\in\mathbb{R}^n : \|x-\bar{x}\|\leq %
\gamma(\mu)
       \right\}.
    \end{equation}
\end{itemize}
If Condition~\ref{assum:simple_jac} holds, and the activation is ReLU, then:
\begin{itemize}
\item The network is Globally
    \emph{incrementally} practically Stable ($\delta$-GpS). In other words, $\forall k\geq0$, given two initial conditions $\{x(0),\ \bar{x}(0)\}$ and the same input $u$, we have:
    \begin{equation}
    		\|x(k)-\bar{x}(k)\|\leq\bar{\rho}^k \|x(0)-\bar{x}(0)\|+  \zeta.
    \end{equation}
     The constant factor is  $\zeta=\frac{\|x(0)-\bar{x}(0)\|}{(1-\bar{\rho})}$.
  \item The network is Globally
    Input-Output \emph{incrementally} practically Stable ($\delta$-IOpS). In other words, given $\{x(0),\ \bar{x}(0)\}$ and two respective inputs $\{u,\  \bar{u}\}$, $\forall k\geq0$ we have:
    \begin{equation}
    		\|x(k)-\bar{x}(k)\|\leq\bar{\rho}^k \|x(0)-\bar{x}(0)\|+ \gamma(\|u-\bar{u}\|) + \zeta.
    \end{equation}
\end{itemize}\label{th:BIBO1}
\end{theorem}

\begin{figure}[t]
    \noindent\begin{minipage}{0.45\columnwidth}
        \begin{algorithm}[H]
        \small
          \caption{Fully Connected  Reprojection}
          \label{alg:fc_rep}
        \begin{algorithmic}
           \STATE {\bfseries Input:} $R\in\mathbb{R}^{\tilde{n}\times {n}}$, $\tilde{n}\leq n$, $\delta = 1-2\epsilon$, $\epsilon\in(0,0.5)$.\vspace{0.2cm}
                   \ \IF{$\|R^TR\|_F>\delta$}
                       \STATE
                       \STATE $\tilde{R}\leftarrow \sqrt{\delta}\frac{R}{\sqrt{\|R^TR\|_F}}$
                   \ \ELSE
                       \STATE $\tilde{R}\leftarrow{R}$
                   \ENDIF
           \STATE {\bfseries Output:} $\tilde{R}$
        \end{algorithmic}
        \end{algorithm}
        \vspace{0.52cm}
    \end{minipage}
    ~~~~~
    \noindent\begin{minipage}{0.45\columnwidth}
        \begin{algorithm}[H]
        \small
          \caption{CNN  Reprojection}
          \label{alg:cnn_rep}
        \begin{algorithmic}
          \STATE {\bfseries Input:} \!$\delta\!\in\mathbb{R}^{N_C}\!\!,C\!\in \mathbb{R}^{n_X\times n_X\times N_C\times N_C}$,
          and $0<\epsilon<\eta<1$.
                   \FOR{each feature map $c$}
                       \STATE $\tilde{\delta}_c \leftarrow \max\bigg(\min\big(\delta_c,1-\eta\big),-1+\eta\bigg)$
                       \STATE $\tilde{C}_{i_\text{centre}}^{c}\leftarrow -1-\tilde{\delta}_c$
                       \IF{$\sum_{j\neq {i_\text{centre}}} \left|C_j^{c}\right| >1-\epsilon-|\tilde{\delta}_c|$}
                           \STATE
                           \STATE $\tilde{C}^{c}_j\leftarrow \left(1-\epsilon-|\tilde{\delta}_c|\right)\frac{C_j^{c} }{\sum_{j\neq {i_\text{centre}}} \left|{C}_j^{c}\right|}$
                       \ENDIF
                   \ENDFOR
               \STATE {\bfseries Output:} $\tilde{\delta}$, $\tilde{C}$
        \end{algorithmic}
        \end{algorithm}
    \end{minipage}
    \vspace{0.2cm}
    \caption{\small Proposed algorithms for enforcing stability.}
\end{figure}
\vspace{-0.2cm}

\section{Implementation}\label{sec:implementation}
In general, an optimization problem with a spectral radius constraint as in
Eq.~(\ref{eq:2norm}) is hard \cite{kanai2017a}. One possible approach is to
relax the constraint to a singular value constraint~\cite{kanai2017a} which is
applicable to both fully connected as well as convolutional layer
types~\cite{yoshida2017a}. However, this approach is only applicable if the
identity matrix in the Jacobian (Eq.~(\ref{eq:jac_gen})) is scaled by a factor
$0 < c < 1$ \cite{kanai2017a}. In this work we instead fulfil the spectral
radius constraint directly.

The basic intuition for the presented algorithms is the fact that for a simple
Jacobian of the form $I + M$, $M \in \mathbb{R}^{ n \times n}$,
Condition~\ref{assum:simple_jac} is fulfilled, if $M$ has eigenvalues with real
part in $(-2, 0)$ and imaginary part in the unit circle. In the supplemental
material we prove that the following algorithms fulfill
Condition~\ref{assum:simple_jac} following this intuition. Note that, in the
following, the presented procedures are to be performed for each block of the
network.

\paragraph{Fully-connected blocks.}
In the fully connected case, we restrict the matrix $A$ to by symmetric and
negative definite by choosing the following parameterization for them:
\begin{equation}
   A = -R^TR-\epsilon I,
\end{equation}
where $R\in\mathbb{R}^{n \times n}$ is trained, and $0<\epsilon\ll1$ is a
hyper-parameter. Then, we propose a bound on the Frobenius
norm, $\|R^TR\|_F$. Algorithm~\ref{alg:fc_rep}, performed during training,
implements the following%
\footnote{The more relaxed condition $\delta\in(0,2)$
is sufficient for Theorem \ref{th:BIBO1}  to hold locally
(supplementary material).}:
\begin{theorem}(Fully-connected weight projection)\label{th:fully} \\
  Given $R \in \mathbb{R}^{n \times n}$, the projection
    $\tilde{R} = \sqrt{\delta}\frac{R}{\sqrt{\|R^TR\|_F}}$, with $\delta=1-2\epsilon \in(0,1)$,
    ensures that $A = -\tilde{R}^{T}\tilde{R} - \epsilon I$ is such that
    Condition \ref{assum:simple_jac} is satisfied for $h \leq 1$ and therefore
    Theorem \ref{th:BIBO1} holds.
\end{theorem}
Note that $\delta=2(1-\epsilon) \in(0,2)$ is also sufficient for stability, however,
the $\delta$ from Theorem \ref{th:fully}
makes the trajectory free from oscillations (critically damped), see Figure \ref{fig:sharedvsunshared}.
 This is further discussed in Appendix.

\paragraph{Convolutional blocks.}
\label{sec:proposedCNNsolution}
The symmetric parametrization assumed in
the fully connected case can not be used for a convolutional layer.
We will instead make use of the following result:
\begin{lemma}\label{lem:CNNconn}
    The convolutional layer Eq. (\ref{eq:CNNmodel}) with zero-padding $p\in\mathbb{N}$,
     and filter size $n_C=2p+1$ has
    a Jacobian of the form Eq.~(\ref{eq:jac_gen}).  with
    ${A} \in \mathbb{R}^{n_X^2N_C \times n_X^2N_C}$.
     The diagonal elements of this matrix, namely, ${A}_{n_X^2c+j, n_X^2c+j}$, $0 \leq c < N_C$,
    $0 \leq j < n_X^2$ are the central elements of the $(c+1)$-th convolutional
    filter mapping $X^{c+1}(k)$, into
     $X^{c+1}(k+1)$,
    denoted by $C_{i_\text{centre}}^{c}$. The other elements
    in row $n_X^2c+j$, $0 \leq c < N_C$, $0 \leq j < n_X^2$ are the remaining
     filter values mapping to $X^{(c+1)}(k+1)$.
\end{lemma}
To fulfill the stability condition, the first step is to set $C_{i_\text{centre}}^{c}=-1-\delta_c$,
where $\delta_c$ is trainable parameter satisfying $|\delta_c|<1-\eta$, and
$0<\eta\ll1$ is a hyper-parameter. Then we will suitably bound the $\infty$-norm
of the Jacobian by constraining the remaining filter elements. The steps are
summarized in Algorithm~\ref{alg:cnn_rep} which is inspired by the Gershgorin
Theorem \cite{Horn:2012:MA:2422911}. The following result is obtained:
\begin{theorem}(Convolutional weight projection)\label{th:implem_conv}\\
 Algorithm~\ref{alg:cnn_rep} fulfils Condition~\ref{assum:simple_jac}
  for the convolutional layer, for $h\leq1$, hence Theorem
  \ref{th:BIBO1} holds.
    \end{theorem}
Note that the algorithm complexity scales with the number of filters.
A simple design choice for the layer is to set $\delta=0$, which
results in $C_{i_\text{centre}}^{c}$ being fixed at $-1$.\footnote{Setting $\delta=0$ removes
the need for hyper-parameter $\eta$ but does not necessarily reduce
conservativeness
as it will further constrain the remaining element of the filter bank.
This is further discussed in the supplementary.}%

\section{Experiments}
\label{sec:experiments}

Experiments were conducted comparing NAIS-Net with ResNet, and variants
thereof, using both fully-connected (MNIST,
section~\ref{sec:analysis}) and convolutional (CIFAR-10/100,
section~\ref{sec:classification}) architectures to quantitatively
assess the performance advantage of having a VDNN where stability
is enforced.

\begin{figure}[t]
\centering
    \noindent
    \begin{minipage}{0.48\columnwidth}
    \begin{figure}[H]
        \centering
            \includegraphics[width=\linewidth]{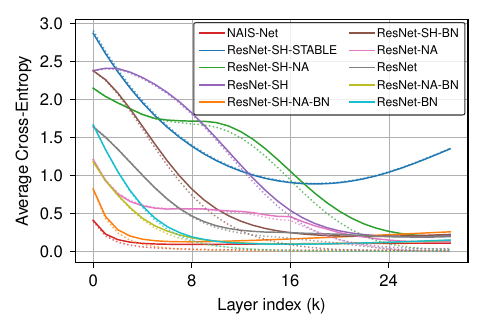}
    \end{figure}
    \end{minipage}
    \noindent
    \begin{minipage}{0.48\columnwidth}
    \begin{figure}[H]
     \includegraphics[width=\linewidth]{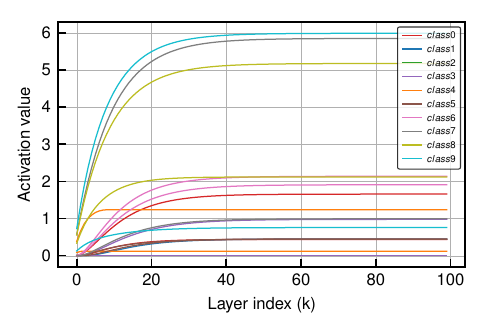}
     \end{figure}
    \end{minipage}
    \caption{
    \small{\bf Single neuron trajectory and convergence. (Left)}
    Average loss of NAIS-Net with different residual
    architectures over the unroll length.
    Note that both {\sc ResNet-SH-Stable} and NAIS-Net satisfy the stability conditions for
    convergence, but only NAIS-Net is able to learn, showing the importance of non-autonomy.
    \small{\bf Cross-entropy loss vs processing depth. (Right)}
    Activation of a NAIS-Net single neuron for
    input samples from each class on MNIST. Trajectories not only differ with respect to
    the actual steady-state but also with respect to the convergence time.}
    \label{fig:sharedvsunshared}
\end{figure}
\subsection{Preliminary Analysis on MNIST}\label{sec:analysis}

\setlist[enumerate]{topsep=0mm,itemsep=0pt,labelindent=0cm,rightmargin=1cm,leftmargin=.5cm}
\newcommand\netitem[1]{\item[]{\sc\bfseries #1:}}

For the MNIST dataset~\cite{lecun1998mnist} a
  single-block NAIS-Net was compared with 9 different $30$-layer ResNet
  variants %
  each with a
  different combination of the following features:
{\bf SH} (shared weights i.e. time-invariant),
{\bf NA} (non-autonomous i.e. input skip connections),
{\bf BN} (with Batch Normalization),
{\bf Stable} (stability enforced by Algorithm~\ref{alg:fc_rep}).
For example, {\sc ResNet-SH-NA-BN} refers to a 30-layer ResNet
that is time-invariant because weights are shared across all layers
(SH), non-autonomous because it has skip connections from the input to
all layers (NA), and uses batch normalization (BN).  Since NAIS-Net is
time-invariant, non-autonomous, and input/output stable
(i.e.\ {\sc SH-NA-Stable}), the chosen ResNet variants represent ablations of the
these three features.  For instance, {\sc ResNet-SH-NA} is a
NAIS-Net without I/O stability being enforced by the reprojection step
described in Algorithm~\ref{alg:fc_rep}, and {\sc ResNet-NA}, is a
non-stable NAIS-Net that is time-variant, i.e\ non-shared-weights, etc.
The NAIS-Net was unrolled for $K=30$ iterations for all input patterns.
All networks were trained using stochastic gradient descent with
momentum $0.9$ and learning rate $0.1$, for 150 epochs.

\paragraph{Results.}
Test accuracy for NAIS-NET was $97.28\%$, while
 {\sc ResNet-SH-BN}  was second best with $96.69\%$, but without
 BatchNorm ({\sc
  ResNet-SH}) it only achieved $95.86\%$  (averaged over
10 runs).

After training, the behavior of each network variant was analyzed
by passing the activation, $x(i)$, though the softmax classifier and
measuring the cross-entropy loss. The loss at each
iteration describes the trajectory of each sample in the latent
space: the closer the sample to the correct steady state the closer
the loss to zero (see Figure~\ref{fig:sharedvsunshared}).
 All variants initially refine their predictions at
each iteration since the loss tends to decreases at each layer,
but at different rates. However, NAIS-Net is the only one that does
so monotonically, not increasing loss as $i$ approaches $30$.
Figure~\ref{fig:sharedvsunshared} shows how neuron activations in NAIS-Net
converge to different steady state activations for different input
patterns instead of all converging to zero as is the case with {\sc
  ResNet-SH-Stable}, confirming the results of~\cite{haber2017}.
Importantly, NAIS-Net is able to
learn even with the stability constraint, showing that non-autonomy is
key to obtaining representations that are stable {\em and} good for learning the task.

NAIS-Net also allows training of unbounded processing depth without
any feature normalization steps.  Note that BN actually
speeds up loss convergence, especially for {\sc
  ResNet-SH-NA-BN} (i.e.\ unstable NAIS-Net). Adding BN makes the
behavior very similar to NAIS-Net because BN also implicitly normalizes
the Jacobian, but it does not ensure that its eigenvalues are in the
stability region.

\subsection{Image Classification on CIFAR-10/100}\label{sec:classification}
Experiments on image classification were performed on standard image
recognition benchmarks CIFAR-10 and CIFAR-100~\cite{krizhevsky2009cifar}.
These benchmarks are simple enough to allow for multiple runs to test for
statistical significance,  yet sufficiently complex to
require convolutional layers.

\paragraph{Setup.}
The following standard architecture was used to compare NAIS-Net with
ResNet\footnote{\url{https://github.com/tensorflow/models/tree/master/official/resnet}}:
three sets of $18$ residual blocks with $16$, $32$, and $64$ filters,
respectively, for a total of $54$ stacked blocks. NAIS-Net was tested in two versions: {\sc NAIS-Net1} where each block is
unrolled just once, for a total processing depth of 108, and {\sc
  NAIS-Net10} where each block is unrolled 10 times per block, for a
total processing depth of 540. The initial learning
rate of $0.1$ was decreased by a factor of $10$ at epochs $150$, $250$
and $350$ and the experiment were run for 450 epochs.  Note that each
block in the ResNet of~\cite{he2015a} has two convolutions (plus
BatchNorm and ReLU) whereas NAIS-Net unrolls with a single convolution.
Therefore, to make  the comparison  of the two architectures as fair
as possible by using the same number of parameters, a single
convolution was also used for ResNet.

\paragraph{Results.}
\begin{figure}[t]%
\noindent\begin{minipage}{0.44\columnwidth}
    \begin{table}[H]\label{tab:cifar}
        \begin{small}
        \begin{sc}
        \begin{tabular}{lcccr}
        \toprule
        Model & CIFAR-10 & CIFAR-100 \\
        & train/test & train/test \\
        \midrule
        \midrule
         ResNet   &99.86$ \pm $0.03  & 97.42 $\pm$ 0.06 \\
                  &91.72$ \pm $0.38  & 66.34 $\pm$ 0.82 \\
        \midrule
        NAIS-Net1  &99.37$ \pm $0.08  & 86.90 $\pm$ 1.47 \\
                   &91.24$ \pm $0.10   & 65.00 $\pm$ 0.52 \\
        NAIS-Net10 &99.50$ \pm $0.02  & 86.91 $\pm$ 0.42 \\
                   &91.25$ \pm $0.46  & 66.07 $\pm$ 0.24 \\
        \bottomrule
        \end{tabular}
        \end{sc}
        \end{small}
        \vspace{1.1cm}

    \end{table}
\end{minipage}
~~~~~~~~
\noindent\begin{minipage}{0.5\columnwidth}
\begin{figure}[H]
    \centering
    \includegraphics[width=\columnwidth]{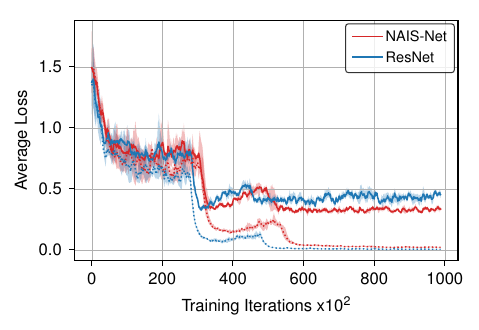}
    \end{figure}
\end{minipage}
    \vspace{-0.5cm}
    \caption{
        \small{\bf CIFAR Results. (Left)}{ Classification accuracy on the CIFAR-10 and CIFAR-100 datasets averaged over 5 runs}.
        \small{\bf Generalization gap on CIFAR-10. (Right)}  Dotted curves
      (training set) are
      very similar for the two networks but NAIS-Net has a
      considerably lower test curve (solid).
  }
  \label{fig:generalizationgap}
\end{figure}
Table~\ref{tab:cifar} compares the performance on the two datasets,
averaged over 5 runs.  For CIFAR-10, NAIS-Net and ResNet performed
similarly, and unrolling NAIS-Net for more than one iteration had
little affect.  This was not the case for CIFAR-100 where {\sc
  NAIS-Net10} improves over {\sc NAIS-Net1} by $1\%$.  Moreover,
although mean accuracy is slightly lower than ResNet, the variance is
considerably lower. Figure~\ref{fig:generalizationgap} shows that
NAIS-Net is less prone to overfitting than a classic ResNet, reducing
the generalization gap by 33\%.  This is a consequence of the stability
constraint which imparts a degree of robust invariance to input
perturbations (see Section \ref{sec:nonautonomousResNet}).  It is also
important to note that NAIS-Net can unroll up to $540$ layers, and still
train without any problems.

\subsection{Pattern-Dependent Processing Depth}\label{sec:depth}
For simplicity, the number of unrolling steps per block in the
previous experiments was fixed.  A more general and potentially more
powerful setup is to have the processing depth adapt automatically.
Since NAIS-Net blocks are guaranteed to converge to a pattern-dependent
steady state after an indeterminate number of iterations, processing
depth can be controlled dynamically by terminating the unrolling
process whenever the distance between a layer representation, $x(i)$,
and that of the immediately previous layer, $x(i-1)$, drops below a
specified threshold.
With this mechanism, NAIS-Net can determine the processing depth for
each input pattern.  Intuitively, one could speculate that similar
input patterns would require similar processing depth in order to be
mapped to the same region in latent space.  To explore this
hypothesis, NAIS-Net was trained on CIFAR-10 with an unrolling
threshold of $\epsilon=10^{-4}$.  At test time the network was
unrolled using the same threshold.

Figure~\ref{fig:histdepth} shows selected images from four
different classes organized according to the final network depth used
to classify them after training.  The qualitative differences seen from
low to high depth suggests that NAIS-Net is using processing depth as
an additional degree of freedom so that, for a given training run, the
network learns to use models of different complexity (depth) for
different types of inputs within each class.  To be clear, the hypothesis is not that
depth correlates to some notion of input complexity where the same images are
always classified at the same depth across runs.

\begin{figure*}[t]
  \centering
  \subfigure[frog]{\includegraphics[width=0.49\textwidth]{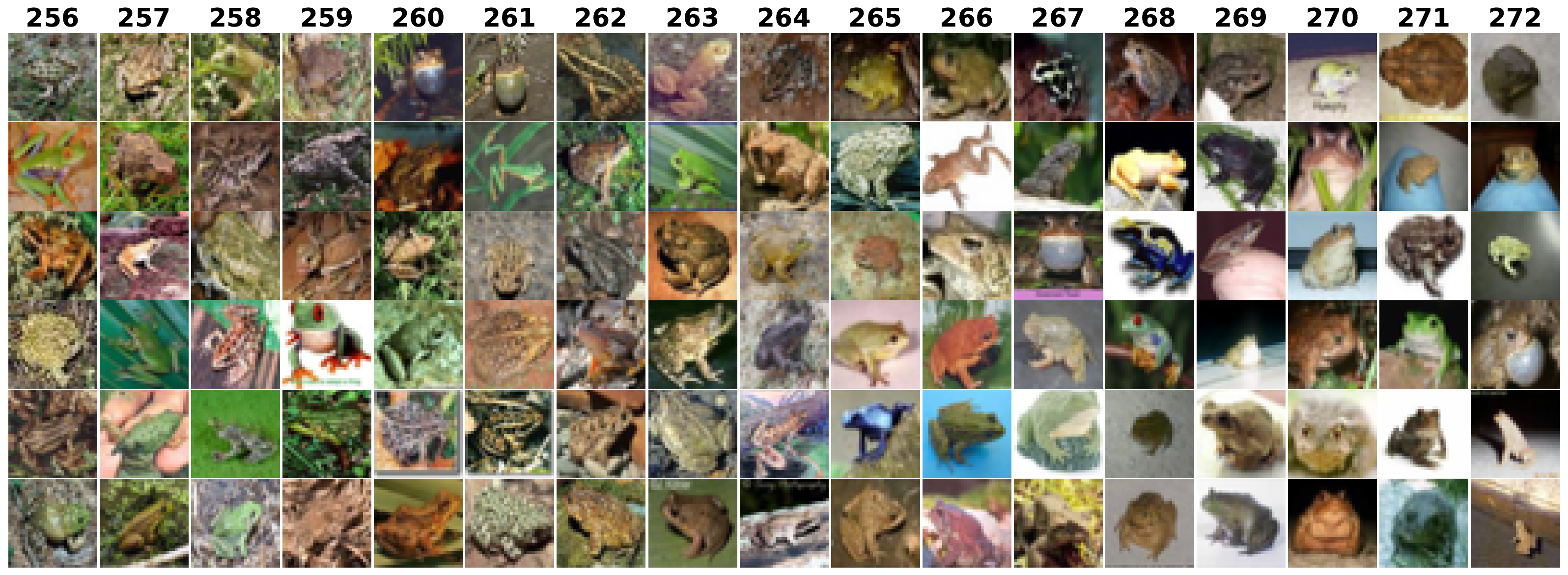}}
    \subfigure[bird]{\includegraphics[width=0.49\textwidth]{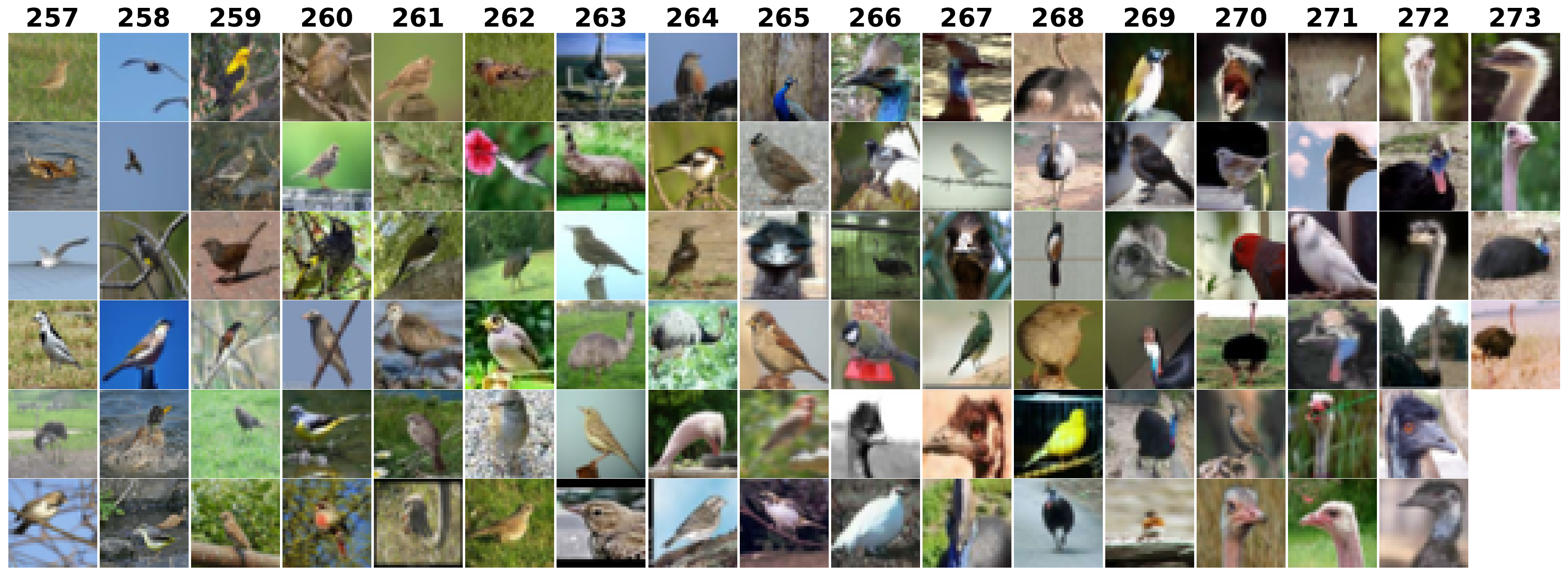}}
    \subfigure[ship]{\includegraphics[width=0.50\textwidth]{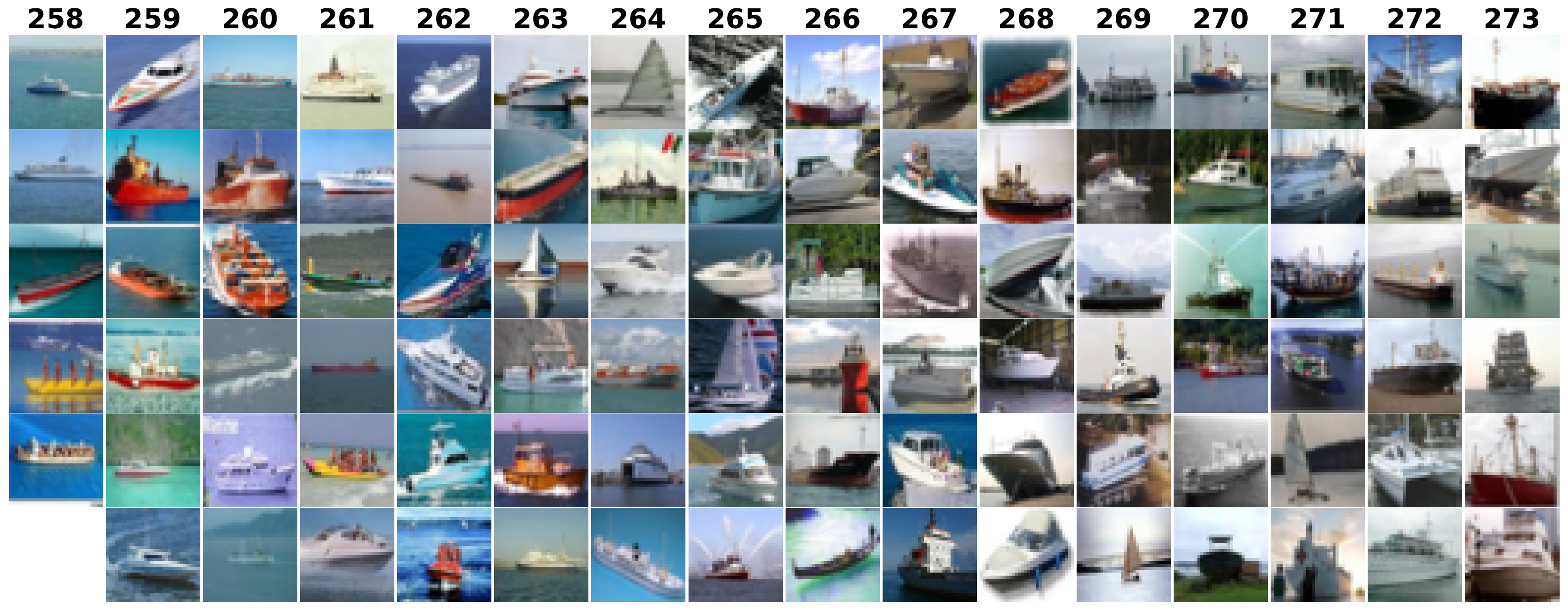}}
    \subfigure[airplane]{\includegraphics[width=0.48\textwidth]{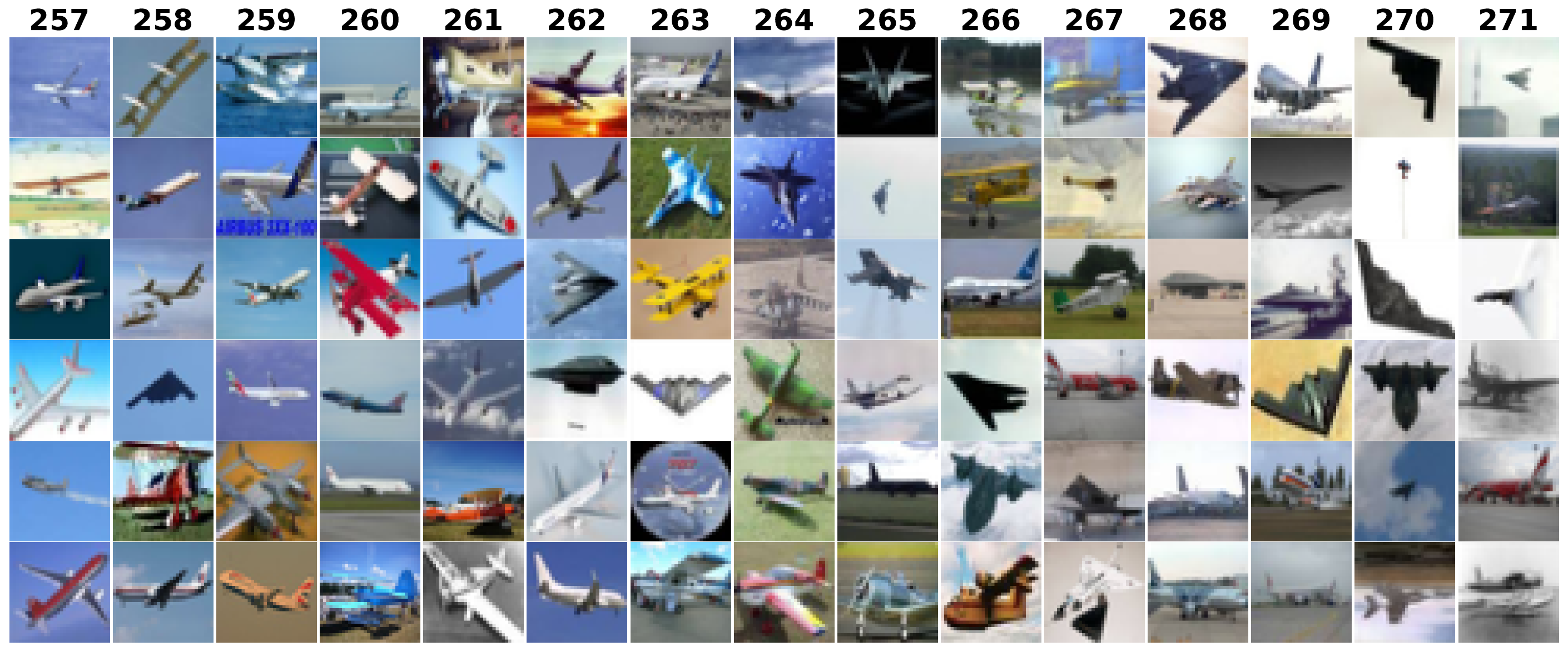}}
    \vspace{-3mm}
    \caption{\small{\bf Image samples with corresponding NAIS-Net depth.} The
      figure shows samples from CIFAR-10 grouped by final network
      depth, for four different classes. The qualitative differences
      evident in images inducing different final depths indicate that
      NAIS-Net adapts processing systematically according
      characteristics of the data.  For example, \emph{``frog''}
      images with textured background are processed with fewer
      iterations than those with plain background. Similarly,
      \emph{``ship''} and \emph{``airplane''} images having a
      predominantly blue color are processed with lower depth than
      those that are grey/white, and \emph{``bird''} images are grouped
      roughly according to bird size with larger species such as
  ostriches and turkeys being classified with greater processing depth.
  A higher definition version of the figure is made available in the supplementary materials.}
    \label{fig:histdepth}
\end{figure*}

\section{Conclusions}
\label{conclusions}

We presented NAIS-Net, a non-autonomous residual architecture that can be
unrolled until the latent space representation converges to a stable
input-dependent state. This is achieved thanks to stability and
non-autonomy properties.
We derived stability conditions for the model and proposed two efficient
reprojection algorithms, both for fully-connected and convolutional layers, to
enforce the network parameters to stay within the set of feasible solutions during
training.

NAIS-Net achieves asymptotic stability and, as consequence of that,
input-output stability. Stability makes the model more robust and we observe a
reduction of the generalization gap by quite some margin,
without negatively impacting performance. The question of scalability to benchmarks such as
ImageNet~\cite{deng2009a} will be a main topic of future work.

We believe that cross-breeding machine learning and control theory
will open up many new interesting avenues for research, and that more
robust and stable variants of commonly used neural networks, both
feed-forward and recurrent, will be possible.

\section*{Aknowledgements}
We want to thank Wojciech Jaśkowski, Rupesh Srivastava and the anonymous reviewers for their
comments on the idea and initial drafts of the paper.
\medskip

\small

\bibliography{urdeq.bib}
\bibliographystyle{plain}

\newpage
\rule[0pt]{\columnwidth}{3pt}
\begin{center}
\huge{\bf{\sc NAIS-Net}: Stable Deep Networks from Non-Autonomous  Differential Equations \\
\emph{Supplementary Material}}
\end{center}
\vspace*{3mm}
\rule[0pt]{\columnwidth}{1pt}%
\vspace*{1in}

\appendix

\section{Basic Definitions for the Tied Weight Case}
Recall, from the main paper, that the stability of a NAIS-Net block with fully connected or convolutional architecture can be analyzed by means of the following vectorised representation:
\begin{equation}\label{eq:DNN_tied}
    x(k+1) = f(x(k),u) = x(k)+h\sigma\bigg(Ax(k)+Bu+b\bigg),
\end{equation}
where $k$ is the unroll index for the considered block. Since the blocks are cascaded,
stability of each block implies stability of the full network. Hence, this supplementary material focuses on theoretical results for a single block.
\subsection{Relevant Sets and Operators}
\subsubsection{Notation}
Denote the slope of the activation function vector, $\sigma(\Delta x(k))$, as the diagonal matrix, $\sigma^{'}(\Delta x(k))$, with entries:
\begin{equation}
   \sigma^{'}_{ii}(\Delta x(k))=\frac{\partial\sigma_{i}(\Delta x(k))}{\partial\Delta x_i(k)}.
\end{equation}

The following definitions will be use to obtain the stability results, where $0<\underline{\sigma}\ll1$:
\begin{equation}
\begin{aligned}
    \mathcal{P}_i = \{(x,u)\in\mathbb{R}^{n}\times \mathbb{R}^{m}:
        \sigma^{'}_{ii}(x,u)\geq\underline{\sigma}\}, \\
    \mathcal{P} = \{(x,u)\in\mathbb{R}^{n}\times \mathbb{R}^{m}:
        \sigma^{'}_{ii}(x,u)\geq\underline{\sigma}, \forall i\}, \\
    \mathcal{N}_i = \mathcal{P} \cup \{(x,u)\in\mathbb{R}^{n}\times \mathbb{R}^{m}:
        \sigma^{'}_{ii}(x,u)\in [0,\underline{\sigma})\}, \\
    \mathcal{N} = \{(x,u)\in\mathbb{R}^{n}\times \mathbb{R}^{m}:
        \sigma^{'}_{ii}(x,u)\in [0,\underline{\sigma}), \forall i\}, \\
\end{aligned}
\end{equation}
In particular, the set $\mathcal{P}$ is such that the activation function is
not saturated as its derivative has a non-zero lower bound.

\subsubsection{Linear Algebra Elements}
The notation, $\|\cdot\|$ is used to denote a \emph{suitable} matrix norm.
This norm will
be characterized specifically on a case by case basis.
The same norm will be used consistently throughout definitions, assumptions and proofs.

We will often use the following:
\begin{lemma}\label{lemma:shift}(Eigenvalue shift)

    Consider two matrices $A\in\mathbb{C}^{n \times n}$,
    and $C=cI+A$ with $c$ being a complex scalar.
    If $\lambda$ is an eigenvalue of $A$ then $c+\lambda$
    is an eigenvalue of $C$.
\end{lemma}
\begin{proof}
Given any eigenvalues $\lambda$ of $A$ with corresponding eigenvector $v$ we have that:
    \begin{equation}
        \lambda v=Av \Leftrightarrow (\lambda+c)v = \lambda v + c v=Av+cv=Av + cIv=(A+cI)v=Cv
    \end{equation}
\end{proof}

Throughout the material, the notation $A_{i\bullet}$ (or equivalently $A_{i}$) is used to denote the $i$-th row of a matrix $A$. %

\subsubsection{Non-autonomuous Behaviour Set}
The following set will be consider throughout the paper:

\begin{definition}(Non-autonomous behaviour set)
The set $\mathcal{P}$ is referred to as the set of \emph{fully non-autonomous}
behaviour in the extended state-input space, and its set-projection over
$x$, namely,
\begin{equation}
    \pi_x(\mathcal{P})=\{x\in\mathbb{R}^n: \exists u\in\mathbb{R}^m,\ (x,u)\in\mathcal{P}\},
\end{equation}
is the set of fully non-autonomous behaviour in the state space. This is the
only set in which every output dimension of the  ResNet with input skip
connection can be directly influenced by the input, given a
non-zero\footnote{The concept of controllability is not introduced here. In the
    case of deep networks we just need $B$ to be non-zero to provide input skip
    connections. For the general case of time-series identification and control,
    please refer to the definitions in \cite{sontag_book}.} matrix $B$.
\end{definition}

Note that, for a $\tanh$ activation, then we simply have that
$\mathcal{P}\subseteq\mathbb{R}^{n+ m}$ (with
$\mathcal{P}\rightarrow\mathbb{R}^{n+ m}$ for
$\bar{\sigma}^{'}\rightarrow0$). For a ReLU activation, on the other hand, for
each layer $k$ we have:
\begin{equation}
    \mathcal{P}=\mathcal{P}(k)=\{(x,u)\in\mathbb{R}^{n}\times \mathbb{R}^{m}:\ A(k)x+B(k)u+b(k)> 0\}.
\end{equation}

\subsection{Stability Definitions for Tied Weights}
This section provides a summary of definitions borrowed from control theory
that are used to describe and derive our main result. The following definitions
 have been adapted from \cite{gallieri_book} and
 refer to the general dynamical
system:
\begin{equation}\label{eq:system}
    x^{+}=f(x,u).
\end{equation}
Since we are dealing with a cascade of dynamical systems (see Figure 1 in (\emph{main paper})),
then stability of the entire network can be enforced by
having stable blocks
\cite{khalil2001}.
In the remainder of this material, we will therefore address a single unroll.
We will cover both the tied and untied weight case,
starting from the latter as it is the most general.

\subsubsection{Describing Functions}
The following functions are instrumental to describe the desired behaviour of
the network output at each layer or time step.

\begin{definition}($\mathcal{K}$-function)
    A continuous function
    $\alpha:\mathbb{R}_{\geq0}\rightarrow\mathbb{R}_{\geq0}$ is said to be a
    $\mathcal{K}$-function ($\alpha\in\mathcal{K}$) if it is strictly
    increasing, with $\alpha(0)=0$.
\end{definition}

\begin{definition}($\mathcal{K}_\infty$-function)
    A continuous function
    $\alpha:\mathbb{R}_{\geq0}\rightarrow\mathbb{R}_{\geq0}$ is said to be a
    $\mathcal{K}_\infty$-function ($\alpha\in\mathcal{K}_\infty$) if it is a
    $\mathcal{K}$-function and if it is radially unbounded, that is $\alpha(r)
    \rightarrow\infty$ as $r\rightarrow\infty$.
\end{definition}

\begin{definition}($\mathcal{KL}$-function)
    A continuous function $\beta:\mathbb{R}^2_{\geq0}\rightarrow\mathbb{R}_{\geq0}$
    is said to be a $\mathcal{KL}$-function ($\beta\in\mathcal{KL}$) if it is a
    $\mathcal{K}$-function in its first argument, it is positive definite and
    non- increasing in the second argument, and if $\beta(r,t)\rightarrow0$ as
    $t\rightarrow\infty$.
\end{definition}

The following definitions are given for time-invariant RNNs, namely DNN with
tied weights. They can also be generalised to the case of untied weights DNN
and time-varying RNNs by considering worst case conditions over the layer
(time) index $k$. In this case the properties are said to hold \emph{uniformly}
for all $k\geq0$. This is done in Section \ref{app:untied}. The tied-weight case follows.

\subsubsection{Invariance, Stability and Robustness}
\begin{definition}(Positively Invariant Set)
    A set $\mathcal{X}\subseteq \mathbb{R}^n$ is said to be positively
    invariant (PI) for a dynamical system under an input
    $u\in\mathcal{U}\subseteq \mathbb{R}^n$ if
    \begin{equation}
        f(x,u)\in\mathcal{X},\ \forall x\in\mathcal{X}.
    \end{equation}
\end{definition}

\begin{definition}(Robustly Positively Invariant Set)
    The set $\mathcal{X}\subseteq \mathbb{R}^n$ is said to be \emph{robustly
    positively invariant} (RPI) to additive input perturbations
    $w\in\mathcal{W}$ if $\mathcal{X}$ is PI for any input $\tilde{u}=u+w,\
    u\in\mathcal{U}, \forall w\in\mathcal{W}$.
\end{definition}

\begin{definition}(Asymptotic Stability)
    \label{def:stability}
    The system Eq. (\ref{eq:system}) is called Globally Asymptotically Stable
    around its equilibrium point $\bar{x}$ if it satisfies the following two
    conditions:
    \begin{enumerate}
        \item \bf{Stability}.
            \normalfont Given any $\epsilon > 0$, $\exists \delta_1 > 0$
            such that if $\|x(t_0)-\bar{x}\| < \delta_1$,
            then $\|x(t)-\bar{x}\| < \epsilon, \forall t > t_0$.
        \item \bf{Attractivity}.
            \normalfont  $\exists \delta_2 > 0$
            such that if $\|x(t_0)-\bar{x}\| < \delta_2$,
            then $x(t) \rightarrow \bar{x}$ as $t \rightarrow \infty$.
    \end{enumerate}
    If only the first condition is satisfied, then the system is \emph{globally
    stable}. If both conditions are satisfied only for some $\epsilon(x(0))>0$
    then the stability properties hold only locally and the system is said to
    be \emph{locally asymptotically stable}.

    Local stability in a PI set $\mathcal{X}$ is equivalent to the existence of
    a $\mathcal{KL}$-function $\beta$ and a finite constant $\delta\geq0$ such
    that:
    \begin{equation}
    	\|x(k)-\bar{x}\|\leq\beta(\|x(0)-\bar{x}\|,k)+\delta,\
        \forall x(0)\in\mathcal{X},\ k\geq0.
    \end{equation}
    If $\delta=0$ then the system is \emph{asymptotically} stable. If the
    positively-invariant set is $\mathcal{X}=\mathbb{R}^n$ then stability holds
    \emph{globally}.
\end{definition}

Define the system output as $y(k)=\psi(x(k))$, where $\psi$ is a continuous,
Lipschitz function. Input-to-Output stability provides a natural extension of
asymptotic stability to systems with inputs or additive
uncertainty\footnote{Here we will consider only the simple case of $y(k)=x(k)$,
therefore we can simply use notions of Input-to-State Stability (ISS).}.

\begin{definition}(Input-Output (practical) Stability) \label{def:IOpS}
    Given an RPI set $\mathcal{X}$, a constant nominal input $\bar{u}$ and a
    nominal steady state $\bar{x}(\bar{u})\in\mathcal{X}$ such that
    $\bar{y}=\psi(\bar{x})$, the system Eq. (\ref{eq:system}) is said to be
    \emph{input-output (practically) stable} to \emph{bounded additive input
    perturbations} (IOpS) in $\mathcal{X}$ if there exists a
    $\mathcal{KL}$-function $\beta$ and a $\mathcal{K}_\infty$ function
    $\gamma$ and a constant $\zeta>0$:

    \begin{eqnarray}
        \|y(k)-\bar{y}\|\leq \beta(\|y(0)-\bar{y}\|,k)+\gamma(\|w\|)+\zeta,\
        \forall x(0)\in \mathcal{X},  \\
        u = \bar{u}+w,\  \bar{u}\in\mathcal{U},\ \forall w\in\mathcal{W},\
        \forall k\geq0. \nonumber
    \end{eqnarray}
\end{definition}

\begin{definition}(Input-Output (Robust) Stability) \label{def:IOS}
    Given an RPI set $\mathcal{X}$, a constant nominal input $\bar{u}$ and a
    nominal steady state $\bar{x}(\bar{u})\in\mathcal{X}$ such that
    $\bar{y}=\psi(\bar{x})$, the system Eq. (\ref{eq:system}) is said to be
    \emph{input-output (robustly) stable} to \emph{bounded additive input
    perturbations} (IOS) in $\mathcal{X}$ if there exists a
    $\mathcal{KL}$-function $\beta$ and a $\mathcal{K}_\infty$ function
    $\gamma$ such that:

    \begin{eqnarray}
        \|y(k)-\bar{y}\|\leq \beta(\|y(0)-\bar{y}\|,k)+\gamma(\|w\|),\
        \forall x(0)\in \mathcal{X},  \\
        u = \bar{u}+w,\  \bar{u}\in\mathcal{U},\ \forall w\in\mathcal{W},\
        \forall k\geq0. \nonumber
    \end{eqnarray}
\end{definition}

\begin{definition}(Input-Output \emph{incremental} Stability) \label{def:deltaIOS}
    Given a pair of initial conditions $\{x(0),\ \bar{x}(0)\}$ and constant inputs $\{u,\ \bar{u}\}$, with $y=x$, system Eq.  (\ref{eq:system}) is said to be Globally 
    \emph{input-output incrementally stable} ($\delta$-IOS) if there exists a
    $\mathcal{KL}$-function $\beta$ and a $\mathcal{K}_\infty$ function
    $\gamma$ such that:

    \begin{eqnarray}
        \|y(k)-\bar{y}(k)\|\leq \beta(\|y(0)-\bar{y}(0)\|,k)+\gamma(\|u-\bar{u}\|), \\ 
        \forall \{x(0),\ \bar{x}(0)\} \in\mathbb{R}^{2n},\   
           \{u,\ \bar{u}\}\in\mathcal{U}^2, 
        \forall k\geq0. \nonumber
    \end{eqnarray}
\end{definition}

\subsection{Jacobian Condition for Stability}
We prove stability by means of the network Jacobian. The $1$-step state transfer Jacobian for tied weights is:
\begin{equation}\label{eq:jactied}
    J(x(k),u)=\frac{\partial f(x(k),u)}{\partial x(k)}=I+h\frac{\partial\sigma(\Delta x(k))}{\partial \Delta x(k)}A=I+h\sigma^{'}\left(\Delta x(k)\right)A.
\end{equation}
Recall, from the main paper, that the following stability condition is introduced:
\begin{condition}(Condition 1 from \emph{main paper})\label{assum:simple_jac_sup}

For any $\underline{\sigma}>0$, the Jacobian
satisfies:
\begin{equation} \label{eq:2norm_sup}
		\bar{\rho}=\sup_{(x,u)\in\mathcal{P}}\rho(J(x,u))<1,
\end{equation}
    where $\rho(\cdot)$ is the spectral radius.
\end{condition}

\subsection{Stability Result for Tied Weights}
Stability of NAIS-Net is described in Theorem 1 from the main paper. For the sake of completeness, a longer version of this theorem is introduced here and will be proven in Section \ref{sec:stab_proof}. Note that proving the following result is equivalent to proving Theorem 1 in the main paper.

\setcounter{theorem}{0}%
\begin{theorem}\label{th:stab1} (Asymptotic stability for shared weights)

If Condition~\ref{assum:simple_jac_sup} holds, then NAIS-Net with
 $tanh$ activation is Asymptotically Stable with respect to
\emph{input dependent} equilibrium points. More formally:
\begin{equation}
    x(k)\rightarrow\bar{x} \in\mathbb{R}^n,\
    \forall x(0)\in\mathcal{X}\subseteq\mathbb{R}^n,\ u\in\mathbb{R}^m.
\end{equation}
The trajectory is described by:
\begin{equation}
		\|x(k)-\bar{x}\|\leq\bar{\rho}^k \|x(0)-\bar{x}\|
\end{equation}
where $\|\cdot\|$ is a suitable matrix norm and $\bar{\rho}<1$ is given in Eq. (\ref{eq:2norm_sup}).

In particular:
\begin{enumerate}
\item The steady state is independent of the initial
state, namely,
$x(k)\rightarrow\bar{x} \in\mathbb{R}^n,\ \forall x(0)\in\mathbb{R}^n$.
The steady state is given by:
    $$\bar{x}=-A^{-1}{(B{u}+b)}. $$
    The network is Globally Asymptotically Stable with respect to $\bar{x}$.

\item The network is Globally
    Input-Output (robustly) Stable for any input perturbation $w\in\mathbb{R}^m$.
    The trajectory is described by:
    \begin{equation}
	      \|x(k)-\bar{x}\|\leq\bar{\rho}^{k} \|x(0)-\bar{x}\|+ \gamma(\|w\|),
    \end{equation}
    with input-output gain
    \begin{equation}\label{eq:ISSgain_tied_sup}
        \gamma(\|w\|)=h\frac{\|{B}\|}{(1-\bar{\rho})} \|w\|.
    \end{equation}
If $\|w\|\leq\mu$,
then the following set is robustly positively invariant:
    \begin{equation}
       {\mathcal{X}}=
       \left\{
           x\in\mathbb{R}^n : \|x-\bar{x}\|\leq %
\gamma(\mu)
       \right\}.
    \end{equation}
      In other words:
      \begin{equation}
	      x(0)\in\mathcal{X}\Rightarrow x(k)\in\mathcal{X},\  \forall k>0.
       \end{equation}
\end{enumerate}

 In the case of ReLU activation:
 \begin{enumerate}
\item The network is Globally
    \emph{incrementally} practically Stable ($\delta$-GpS). In other words, $\forall k\geq0$, given two initial conditions $\{x(0),\ \bar{x}(0)\}$ and the same input $u$, we have:
    \begin{equation}
    		\|x(k)-\bar{x}(k)\|\leq\bar{\rho}^k \|x(0)-\bar{x}(0)\|+  \zeta.
    \end{equation}
      The constant factor is $\zeta=\frac{\|x(0)-\bar{x}(0)\|}{(1-\bar{\rho})}$. 
  \item The network is Globally
    Input-Output incrementally practically Stable ($\delta$-IOpS). In other words, given $\{x(0),\ \bar{x}(0)\}$ and two respective inputs $\{u,\  \bar{u}\}$, $\forall k\geq0$ we have:
    \begin{equation}
    		\|x(k)-\bar{x}(k)\|\leq\bar{\rho}^k \|x(0)-\bar{x}(0)\|+ \gamma(\|u-\bar{u}\|) + \zeta.
    \end{equation}
    The input-output gain is given by Eq. (\ref{eq:ISSgain_tied_sup}). 
    
Therefore, given $u=\bar{u}+w$ and $\bar{u}$ being the nominal (or training) input, then the network state delta converges to:
    \begin{equation}
        \|x-\bar{x}\|\leq \frac{\|x(0)-\bar{x}(0)\|}{(1-\bar{\rho})}+ \gamma( \mu).
    \end{equation}
    
\end{enumerate}\label{th:BIBO1_sup}
\end{theorem}

\section{NAIS-Net with Untied Weights}\label{app:untied}

\subsection{Proposed Network with Untied Weights}
The proposed network architecture with skip connections and our robust
stability results can be extended to the untied weight 
 case. In particular, a single NAIS-Net block is analysed, where the weights
  are not shared throughout the unroll.   

\subsubsection{Fully Connected Layers}
Consider in the following Deep ResNet with input skip connections and untied
weights:
\begin{equation}\label{eq:DNN_untied}
    x(k+1) = f(x(k),u,k) = x(k)+h\sigma\bigg(A(k)x(k)+B(k)u+b(k)\bigg),
\end{equation}
where $k$ indicates the layer, $u$ is the input data, $h>0$, and $f$ is a
continuous, differentiable function with bounded slope. The activation operator
$\sigma$ is a vector of (element-wise) instances of a non-linear activation
function.
In the case of tied weights, the DNN Eq. (\ref{eq:DNN_untied}) can be seen as a
finite unroll of an RNN, where the layer index $k$ becomes a time index and the
input is passed through the RNN at each time steps. This is fundamentally a
linear difference equation also known as a discrete time dynamic system. The
same can be said for the untied case with the difference that here the weights
of the RNN will be time varying (this is a time varying dynamical system).

\subsubsection{Convolutional Layers}
For convolutional networks, the proposed layer architecture can be extended as:
\begin{equation} \label{eq:CNNmodel_untied}
    \begin{aligned}
    X(k+1) & = F(X(k),U,k) \\
           & = X(k)+h\sigma\bigg(C(k)*(k)+D(k)*U+E(k)\bigg),
    \end{aligned}
\end{equation}
where $X^{(i)}(k)$, is the layer state matrix for channel $i$, while $U^{(j)}$,
is the layer input data matrix for channel $j$ (where an appropriate zero
padding has been applied) at layer $k$. An equivalent representation to
Eq. (\ref{eq:CNNmodel_untied}), for a given layer $k$, can be computed in a similar
way as done for the tied weight case in Appendix~\ref{app:CNN}. In
particular, denote the matrix entries for the filter tensors $C(k)$ and
$D(k)$ and $E(k)$  as follows: $C^{(c)}_{(i)}(k)$ as the state convolution
filter from state channel $i$ to state channel $c$, and  $D^{(c)}_{(j)}(k)$ is
the input convolution filter from input channel $j$ to state channel $c$, and
$E^{(c)}(k)$ is a bias matrix for the state channel $c$.

Once again, convolutional layers can be analysed in a similar way to fully
connected layers, by means of the following vectorised representation:
\begin{equation}\label{eq:vectorisedCNN_untied}
    \begin{aligned}
    x(k+1) = x(k)+h\sigma\big(A(k)x(k)+B(k)u+b(k)\big).
    \end{aligned}
\end{equation}
By means of the vectorised representation Eq. (\ref{eq:vectorisedCNN_untied}), the
theoretical results proposed in this section will hold for both fully connected
and convolutional layers.

\subsection{Non-autonomous set}
Recall that, for a $\tanh$ activation, for each layer $k$ we have a different
set $\mathcal{P}(k)\subseteq\mathbb{R}^{n+ m}$ (with
$\mathcal{P}(k)=\mathbb{R}^{n+ m}$ for $\epsilon\rightarrow0$). For ReLU
activation, we have instead:
\begin{equation}
    \mathcal{P}=\mathcal{P}(k)=\{(x,u)\in\mathbb{R}^{n}\times \mathbb{R}^{m}:\ A(k)x+B(k)u+b(k)> 0\}.
\end{equation}

\subsection{Stability Definitions for Untied Weights}

For the case of untied weights, let us consider the origin as a reference point
($\bar{x}=0,\ \bar{u}=0$) as no other steady state is possible without assuming
convergence for $A(k),\ B(k)$. This is true if $u=0$ and if $b(k)=0, \forall
k\geq\bar{k}\geq0$. The following definition is given for stability that is
verified \emph{uniformly} with respect to the changing weights:

\begin{definition}(Uniform Stability and Uniform Robustness)
    Consider $\bar{x}=0$ and $\bar{u}=0$. The network origin is said to be
    \emph{uniformly} asymptotically or simply uniformly stable and,
    respectively, uniformly practically stable (IOpS), uniformly Input-Output
    Stable (IOS) or uniformly incrementally stable ($\delta$-IOS) if, respectively, Definition~\ref{def:stability},
    \ref{def:IOpS}, \ref{def:IOS} and \ref{def:deltaIOS} hold with a unique set of describing
    functions, $\beta,\ \gamma,\ \zeta$ for all possible values of the layer
    specific weights, $A(k),\ B(k),\ b(k),\ \forall k\geq0$.
\end{definition}

\subsection{Jacobian Condition for Stability}
The state transfer Jacobian for untied weights is:
\begin{equation}\label{eq:jacuntied}
    J(x(k),u,k)=\frac{\partial f(x(k),u,k)}{\partial x(k)}=I+h\sigma^{'}\left(\Delta x(k)\right)A(k).
\end{equation}
The following assumption extends our results to the untied weight case:
\begin{condition}\label{assum:Jacobian_untied}
    For any $\bar{\sigma}^{'}>0$, the Jacobian satisfies:
    \begin{equation}\label{eq:stability_untied}
        \bar{\rho}=\sup_{(x,u)\in\mathcal{P}}\sup_k\rho\left(J(x(k),u,k)\right)<1, 
    \end{equation}
    where $\rho(\cdot)$ is the spectral radius.
\end{condition}
Condition Eq. (\ref{eq:stability_untied}) can be enforced during training for each
layer using the procedures presented in the paper. %

\subsection{Stability Result for Untied Weights}
Recall that we have taken the origin as the reference equilibrium point,
namely, $\bar{x}=0$ is a steady state if $\bar{u}=0$ and if $b(k)=0, \forall
k\geq\bar{k}\geq0$. Without loss of generality, we will assume $b(k)=0, \forall
k$ and treat $u$ as a disturbance, $u=w$, for the robust case. The following
result is obtained:

\begin{theorem}(Main result for untied weights)
    \label{th:BIBO2}

    If Condition \ref{assum:Jacobian_untied} holds, then NAIS--net with untied weights and with $\tanh$
    activation is Globally Uniformly Stable. In other words there is a set
    $\bar{\mathcal{X}}$ that is an ultimate bound, namely:
    \begin{equation}
        x(k)\rightarrow\bar{\mathcal{X}}\subseteq\mathbb{R}^n,\ \forall
        x(0)\in\mathcal{X}\subseteq\mathbb{R}^n.
    \end{equation}
    The describing functions are:
            \begin{equation} \label{eq:describing_untied}
                \begin{aligned}
                    &\beta(\|x\|,k)=\bar{\rho}^k \|x\|,
                    &\gamma(\|w\|)=h\frac{\|\bar{B}\|}{(1-\bar{\rho})} \|w\|,\\
                    & \bar{B}=\sup_k\|B(k)\|,\ \bar{\rho}<1,
                \end{aligned}
            \end{equation}
            where $\|\cdot\|$ is the matrix norm providing the tightest bound 
            to the left-hand side of Eq. (\ref{eq:stability_untied}), 
            where $\bar{\rho}$ is defined.
    
    In particular, we have:
    \begin{enumerate}
        \item If the activation is $\tanh$ then the network is Globally
            Uniformly Input-Output robustly Stable for any input perturbation
            $w\in\mathcal{W}=\mathbb{R}^m$. 
            Under no input actions, namely if $u=0$, then the origin is Globally
            Asymptotically Stable.
            If $u=w\in\mathcal{W}$ where $\mathcal{W}$ is the norm ball of
            radius $\mu$, then the following set is RPI:
            \begin{equation}
                {\mathcal{X}}=\left\{
                    x\in\mathbb{R}^n : \|x\|\leq \bar{r}=\frac{h\|\bar{B}\|}{1-\bar{\rho}}\mu
                \right\}.
            \end{equation}

        \item If the activation is ReLu then the network is Globally Uniformly Incrementally IO Stable for
            input perturbations in a compact $\mathcal{W}\subset\mathbb{R}^m$.
            The describing functions  are given by Eq. (\ref{eq:describing_untied}) and
             the constant term is $\zeta=\frac{\bar{r}}{(1-\bar{\rho})}$, 
            where $\bar{r}$ is the norm ball radius for the initial
            condition, namely,
            \begin{equation}
                \mathcal{X}=\{x\in\mathbb{R}^n: \|x(0)-\bar{x}(0)\|\leq\bar{r}\}.
            \end{equation}

            If $u=w\in\mathcal{W}$ where $\mathcal{W}$ is the norm ball of
            radius $\mu$, then the state delta-converges to the following ultimate bound:
            \begin{equation}
                \bar{\mathcal{X}} = \left\{
                 x \in\mathbb{R}^n: \|x-\bar{x}\|\leq \zeta+h\frac{\|\bar{B}\| \mu}{(1-\bar{\rho})}
                \right\}.
            \end{equation}
    \end{enumerate}
\end{theorem}

Note that, if the network consists of a combination of fully connected and convolutional layers, then a single norm inequality with corresponding $\beta$ and $\gamma$ can be obtained by means of matrix norm identities. For instance, since for any norm we have that $\|\cdot\|_q\leq\alpha\|\cdot\|_p$, with $\alpha>0$, one could consider the global describing function $\beta(\cdot,\cdot) = \alpha\beta_p(\cdot,\cdot)$. Similarly for $\gamma$.

\section{Stability Proofs}\label{sec:stab_proof}
Stability of a system can then be assessed for instance by use of the so-called
Lyapunov indirect method \cite{strogatz2014}, \cite{khalil2001}, \cite{sontag_book}.
Namely, if the linearised system
around an equilibrium is stable then the original system is also stable. This
also applies for linearisations around all possible trajectories if there is a
single equilibrium point, as in our case. In the following, the bias term is
sometimes omitted without loss of generality (one can add it to the input).
Note that the proofs are also valid for RNNs with varying input sequences
$u(k)$, with asymptotic results for converging input sequences,
$u(k)\rightarrow \bar{u}$. Recall that $y(k)=x(k)$
by definition and let's consider the case of $b(k)=0,\ \forall k$, without loss
of generality as this can be also assumed as part of the input. 

The untied case is covered first. 
The proposed results make use of the $1$-step state transfer Jacobian Eq. (\ref{eq:jacuntied}). 
Note that this is not the same as the full input-to-output Jacobian, for
instance as the one defined in~\cite{duvenaud2014}. The full Jacobian will also contain the input to state map,
given by Eq. (\ref{eq:conv}), which does not affect stability.  The input to state map will be used later on to
investigate robustness as well as the asymptotic network behaviour. First, we
will focus on stability, which is determined by the 1-step state transfer map. 
For the sake of brevity, denote the layer $t$ state Jacobian from Eq. (\ref{eq:jacuntied}) as:
$${J}(t)=I+h\sigma^{'}(\Delta x(t))A(t),\ \forall t\geq0.$$
Define the discrete time convolution sum as:
$$y_u(k)=\sum_{t=0}^{k-1} H(k-t)u(t),\ k>0,\ y_u(0)=0.$$
The above represent the \emph{forced response} of a linear time invariant (LTI)
system, namely the response to an input from a zero initial condition, where
$H$ is the (in the LTI case stationary) impulse response. Conversely, the
\emph{free response} of an autonomous system from a non-zero initial condition
is given by:
$$y_{x_0}(k)=\left(\prod^{k-1}_{t=0} J(t)\right)x(0).$$
The free response tends to zero for an asymptotically stable system.
Considering linearised dynamics allows us to use superposition, in other words,
the linearised system response can be analysed as the sum of the free and
forced response:
$$y(k)=y_{x_0}(k)+y_u(k).$$
Note that this is not true for the original network, just for its linearisation.

For the considered network, the forced response of the linearised (time
varying) system to an impulse at time $t$, evaluated at time $k$, is given by:
\begin{eqnarray}
    H(k,t)=h\sigma^{'}(\Delta x(t))B(t),\ \text{if}\ k=t+1, \\
    H(k,t)=\left(\prod^{k-2}_{l=t} J(l+1)\right)h\sigma^{'}(\Delta x(t))B(t),\ \forall k\geq t+2.
\end{eqnarray}
Therefore, the forced response of the linearised system is:
\begin{eqnarray}\label{eq:conv}
    y_u(k) = h\sigma^{'}(x(k-1),u(k-1))B(k-1)u(k-1)\nonumber \\
           + \sum_{t=0}^{k-2}\left(\prod^{k-2}_{l=t} J(l+1)\right)h\sigma^{'}(\Delta x(t))B(t)u(t)\nonumber\\
           = \sum_{t=0}^{k-1}H(k,t)u(t).
\end{eqnarray}
Note that:
$$\|H(k,t)\|\leq h \sup_{(x,u) \in\mathcal{P},\ j}\|J(x,u,j)\|^{(k-t-1)}\|B(j)\|,\ \forall k,t\geq0$$
since $\|\sigma^{'}\|\leq1$.

To prove our results, we will use the fact that the network with $\tanh$ or
ReLU activation is globally Lipschitz, and that the activation functions have
Lipschitz constant $M=1$ with respect to any norm. This follows from the fact that:
$$\sup_{\Delta x\in\mathbb{R}^n}\max_{i}|\sigma^{'}_{ii}(\Delta x)|=1.\ $$
This means that the trajectory norm can be upper bounded inside
$\mathcal{P}\subseteq\mathbb{R}^{n}\times \mathbb{R}^{m}$ by means of:
\begin{equation}\label{eq:upper bound}
    \|f(x,u,k) \| \leq \sup_{(x,u)\in\mathcal{P}} \sup_k\|J(x,u,k)\| \|x\|+\sup_{(x,u)\in\mathcal{P}}\sup_k \sum_{t=0}^{k-1}\|H(k,t)\| \|u\|
\end{equation}
For a non-zero steady state $\bar{x}$ or a trajectory $\bar{x}(k)$ starting from a different initial condition $\bar{x}(0)$, we can define the error function:
\begin{equation}
e(k)=x(k)-\bar{x}(k),
\end{equation}
From the fact that the network equation is globally Lipschitz and that the steady states satisfy 
$f(\bar{x},\bar{u})=\bar{x}$, we can use a single statement to show \emph{convergence} for both architectures, respectively, with respect to a steady-state for $\tanh$ or to any other trajectory for ReLU. We will show that the function that maps $e(k)$
into $e(k+1)$ is also Lipschitz for all $k$. In particular, denoting $e(k+1)$
as simply $e^{+}$ and dropping the index $k$ for the sake of notation, we have
the following:
\begin{equation}
    \begin{aligned}
    \|e^+\| &= \|f(x,\bar{u})-\bar{x}\| = \|f(x,\bar{u})-f(\bar{x},\bar{u})\|  \\
     &=\|x+h\sigma(Ax+B\bar{u}+b)-\bar{x}-h\sigma(A\bar{x}+B\bar{u}+b)\|  \\
      &\leq \sup_{(x,u)\in\mathcal{P}}\sup_t\left(\|I+h\sigma^{'}(\Delta x(t))A\| \|x-\bar{x} \|\right) + h \|B\bar{u}-B\bar{u}+b-b\| \\
      &= \sup_{(x,u)\in\mathcal{P}}\sup_t\left(\|I+h\sigma^{'}(\Delta x(t))A\|\right)  \|x-\bar{x} \|.
    \end{aligned}
\end{equation}
Note that, at the next step, we also have:
\begin{eqnarray}
    \|f(f(x,\bar{u}),\bar{u})-\bar{x}\|=\|f(f(x,\bar{u}),\bar{u})-f(f(\bar{x},\bar{u}),\bar{u})\| \nonumber \\
    \leq \sup_{(x,u)\in\mathcal{P}}\sup_t\left(\|I+h\sigma^{'}(\Delta x(t))A\|\right)  \|f(x,\bar{u})-f(\bar{x},\bar{u})\| \nonumber \\ + h \|B\bar{u}-B\bar{u}+b-b\| \nonumber \\
    = \sup_{(x,u)\in\mathcal{P}}\sup_t\left(\|I+h\sigma^{'}(\Delta x(t))A\|\right)  \|f(x,\bar{u})-f(\bar{x},\bar{u})\| \nonumber \\
    = \sup_{(x,u)\in\mathcal{P}}\sup_t\left(\|I+h\sigma^{'}(\Delta x(t))A\|\right)  \|f(x,\bar{u})-\bar{x} \| \nonumber \\
    \leq  \sup_{(x,u)\in\mathcal{P}}\sup_t\left(\|I+h\sigma^{'}(\Delta x(t))A\|^2\right)  \|x-\bar{x} \|.
\end{eqnarray}
and therefore, by induction it also follows that the trajectory at time $k$
satisfies:
\begin{eqnarray}
    \|\circ^{k}_{t=0}f(x,\bar{u})-\bar{x} \|=\|f \circ f\circ \dots f \left( f(x,\bar{u})\right)-\bar{x}\| \nonumber \\
    \leq \left(\sup_{(x,u)\in\mathcal{P}}\sup_t\|I+h\sigma^{'}(\Delta x(t))A\|\right) ^k \|x-\bar{x} \| =\bar{\rho}^k \|x-\bar{x} \| \label{eq:lipschits_evo}
\end{eqnarray}
By definition of $\mathcal{P}$, from the above we have that, for $\tanh$ activation, 
the network trajectory with respect to an equilibrium point
$\bar{x}$, can be upper bounded in norm by the trajectory of the linearization while in $\mathcal{P}$.
In the limit case of $\bar{\sigma}^{'}\rightarrow 0$ the bound becomes \emph{global} as  
$\mathcal{P}\rightarrow\mathbb{R}^{n}\times \mathbb{R}^{m}$. 
On the other hand, for ReLU activation the upper bounds are valid only \emph{locally},
in $\mathcal{P}$ iself. These considerations will be used to prove our main results.

\begin{proof}(Proof of Theorem \ref{th:BIBO2}: Main result for untied weights)
In the untied weight case the network does not admit non-zero steady states, as
the matrices $A(k)$ and $B(k)$ are not assumed to converge with increasing
$k$. Let us therefore consider the origin as our reference point, namely,
$(\bar{x}=0,\bar{u}=0)$. Therefore, for the robust results we will consider
the input $u=w$. The proof can now proceed by norm bounding the
superposition of the free and forced response. Recall that $y(k)=x(k)$ by
definition and consider $b(k)=0,\ \forall k$, without loss of generality.
Now, if $(x,u) \in\mathcal{P}$, for the linearised system we have the
following:
\begin{equation}
    \label{eq:ISS}
    \begin{aligned}
        \|x(k)-\bar{x}\| &= \|e(k)\| \\
        &\leq
         \sup_{(x,u)\in\mathcal{P}}\sup_j\|{J}(x,u,j)\|^k\|e(0)\|
        +\sum_{t=0}^{k-1}\left\|H(k,t)\right\| \|w(t)\|  \\
        &\leq \bar{\rho}^k\|e(0)\| + h\sum_{t=0}^{k-1} \bar{\rho}^{k-t-1}\ \sup_j\|B(j)\|\|w\| \\
        &\leq  \beta(\|e(0)\|,k)+\gamma(\|w\|).
    \end{aligned}
\end{equation}
In the above, we have defined:
\begin{equation}
    \begin{aligned}
        &\beta(\|e\|,k)=\bar{\rho}^k \|e\|,\\
        &\gamma(\|w\|)=h\frac{\|\bar{B}\|}{(1-\bar{\rho})} \|w\|,\\
        & \bar{B}=\sup_j\|B(j)\|,\ \bar{\rho}<1.
\end{aligned}
\end{equation}
In Eq. (\ref{eq:ISS}), we have used the fact that if $\rho(J)<1$ then
there exist a suitable matrix norm $\|\cdot\|$ and a constant $\bar{\rho}<1$ such that
$\|J\|\leq\bar{\rho}$. This stems directly from Theorem 5.6.12, page 298 of \cite{Horn:2012:MA:2422911}. 
In our case, $\bar{\rho}<1$ is verified when $(x,u) \in\mathcal{P}$ since, from
Condition \ref{assum:Jacobian_untied}, we have that $\sup_{(x,u)\in\mathcal{P}}\sup_j\rho\left({J}(x,u,j)\right)<1$.  
Outside the region $\mathcal{P}$, however, we need to consider the specific activation functions: for $tanh$, the region $\mathcal{P}$ can be taken to be any subset of the reals, therefore being outside this set as 
$\mathcal{P}\rightarrow \mathbb{R}^{n\times m}$ contradicts asymptotic stability. 
\footnote{Note that, when using $\tanh$, in the limit case of $\bar{\sigma}^{'}=0$, 
$\mathcal{P}=\mathbb{R}^{n+ m}$
the system becomes simply stable. This is a small subtlety in our result for $\tanh$, which could be 
defined as \emph{almost global} or more formally \emph{valid in every bounded subset of reals}.
However, from the steady state analysis (to follow) and the contraction result 
we can see that 
${\sigma}^{'}=0$ is highly unlikely to happen in practise.}  
For ReLU activation, being outside $\mathcal{P}$ means that (at least part of) the system is autonomous
and therefore the network is 
simply stable, as well as incrementally Input-Output practically stable. 

Theorem statements are proven as follows:
\begin{enumerate}

\item Note that, the condition $(x(t),u)\not\in\mathcal{P}$  is only possible
    for ReLU activations, since for $tanh$ activation we have that
    $\mathcal{P}\rightarrow\mathbb{R}^{n+ m}$ for $\bar{\sigma}^{'}\rightarrow0$ and this would 
    contradict stability result Eq. (\ref{eq:ISS}) inside the set.
    This means that Eq. (\ref{eq:ISS}) holds globally and therefore the considered
    network with tanh activations is Input-Output stable for a real-valued
    input $u$.
    Same considerations apply for the robust case with an additive perturbation
    $w\in\mathbb{R}^m$.

    In order to show the existence of a robust positively invariant set, consider the candidate set:
    \begin{equation}
        {\mathcal{X}}=\left\{x\in\mathbb{R}^n : \|e\|=\|x-\bar{x}\|\leq \bar{r}=\frac{r}{1-\bar{\rho}}\right\},
    \end{equation}
    and the disturbance set:
    \begin{equation}
        {\mathcal{W}}=\{w\in\mathbb{R}^m : \|w\|\leq \mu\}.
    \end{equation}
    Then from the IOS inequality we have that, if $x(0)\in\mathcal{X}$ and
    $w\in\mathcal{W}$, the bound $\mu$ can be computed so that $\mathcal{X}$ is
    RPI. To construct the bound, it is sufficient to have the following
    condition to hold $\forall k\geq0$:
    \begin{equation}
        \|e(k)\|\leq \bar{\rho}^{k} \|e(0)\| + \frac{h\|B\|}{(1-\bar{\rho})} \|w\|  \leq \bar{r}, \ \forall w:\ \|w\|\leq\mu.
    \end{equation}
    The above is verified $\forall k\geq0$ if the bound $\mu$ satisfies by the
    following sufficient condition when $\|e(0)\|\geq\bar{r}$:
    \begin{equation}
        \begin{aligned}
            \|e(k)\|&\leq \bar{\rho}^{k-1} \|e(0)\| + \bar{\rho}^{k-1} (\bar{\rho}-1) \|e(0)\| + \frac{h\|\bar{B}\|}{(1-\bar{\rho})} \|w\|   \\
            &\leq \bar{\rho}^{k-1} \|e(0)\| + (\bar{\rho}-1) \|e(0)\| \sup_j \sum_{t=0}^{j} \bar{\rho}^{j-1} + \frac{h\|\bar{B}\|}{(1-\bar{\rho})} \|w\|   \\
            &= \bar{\rho}^{k-1} \|e(0)\| - \frac{(1-\bar{\rho})}{(1-\bar{\rho})}\|e(0)\| + \frac{h\|\bar{B}\|}{(1-\bar{\rho})} \|w\|   \\
            &\leq \bar{\rho}^{k-1} \|e(0)\|\leq \|e(0)\|  \\
            &\leq \bar{\rho}^{k-1} \bar{r} \leq \bar{r},\ \text{IF}\ \|e(0)\|=\bar{r}  \\
            &\Leftarrow   \frac{h\|\bar{B}\|}{(1-\bar{\rho})} \|w\|\leq\|e(0)\| \geq \bar{r} \\
            &\Leftarrow  \frac{r}{(1-\bar{\rho})}  \geq \frac{h\|\bar{B}|_2}{(1-\bar{\rho})} \|u\|  \\
            &\Leftrightarrow r \geq {h\|\bar{B}\|} \|w\|  \\
            &\Leftarrow  \|w\| \leq \frac{r}{h\|\bar{B}\|}=\mu.
        \end{aligned}
    \end{equation}
    This will not necessarily hold in the interior of the set but it will hold
    outside and on the boundary of this compact set. Therefore the set
    $\mathcal{X}$ is invariant under $u=\bar{u}+w$, with $w\in\mathcal{W}$,
    namely, no solution starting inside $\mathcal{X}$ can pass its boundary
    under any $w\in\mathcal{W}$.  Note that, for $\tanh$ activations, this
    result holds globally. Conversely, given a bound $\mu$ for $\mathcal{W}$,
    we can compute $\bar{r}$ such that there is a set $\mathcal{X}$ that is
    RPI, namely:
    \begin{equation}
        {\mathcal{X}}=\left\{x\in\mathbb{R}^n : \|x\|\leq \bar{r}=\frac{h\|\bar{B}\|}{1-\bar{\rho}}\mu\right\}.
    \end{equation}
    Global practical Stability follows by    taking $\bar{x}=0$ and $\delta=\gamma(\|u\|)$.

    \item For the ReLU activation, the set $\mathcal{P}(k)$ does not cover the
        entire $\mathbb{R}^n$. This complicates the analysis further as the output
        $i$ of the network layer $k$ becomes autonomous when $(x(k),u)\not\in
        \mathcal{P}_i(k)$.
        In particular, if at any $k=t$ we have $(x(t),u) \not \in
        \mathcal{P}_i(t)$ then because of linearity we also have that:
        
\begin{equation}
    \begin{aligned}
        \ x_i(t+1) = x_i(t)\Rightarrow  \|e(t+1)\| &\leq \beta(\|e(0)\|,t+1)+\gamma(\|w\|)+\|e_i(t)-(I_i+A_i) e(t)\| \nonumber \\
        &\leq\beta(\|e(0)\|,t+1)+\gamma(\|w\|)+\|e(t)-(I+A) e(t)\| \nonumber \\
        &=\beta(\|e(0)\|,t+1)+\gamma(\|w\|)+\|-A e(t)\| \nonumber \\
        &\leq\beta(\|e(0)\|,t+1)+\gamma(\|w\|)+\zeta_t
    \end{aligned}
\end{equation}
where:
\begin{equation}
    \zeta_t=\|A\|\|e(t)\|\leq\|A\|\sup_{x(0)\in{\mathcal{X}}}\bigg(\beta(\|e(0)\|,t) \bigg)=\bar{\rho}^t\|A\| \sup_{x(0)\in{\mathcal{X}}}\|e(0)\|.
\end{equation}

Where ${\mathcal{X}}=\{x\in\mathbb{R}^n : \|x-\bar{x}\|\leq \bar{r}\}$ is a
bounded set of initial states. From the last two equations and the fact that $\bar{\rho}$ is the spectral radius of $(I+A)$ we can satisfy the 
$\delta$-IOpS condition, with:
\begin{eqnarray}
    \label{eq:IOpS_cond}
    \|e(k)\|\leq \beta(\|e(0)\|,k)+\gamma(\|w\|)+\sum_{t< k} \bar{\rho}^t \zeta_t, \\
    \zeta_t\leq\zeta=\sup_k\sum_{j=0}^k\zeta_j=\frac{\|A\|}{(1-\bar{\rho})}\bar{r}\leq \frac{\|A\|}{\|I-I-A\|}\bar{r}=\bar{r}.
\end{eqnarray}
 If $u=\bar{u}+w$ with
$w\in\mathcal{W}=\{w\in\mathbb{R}^m: \|w\|\leq \mu\}$ then, by taking the sup
over this set of the gain $\gamma$, and over $k$ of the inequality
Eq. (\ref{eq:IOpS_cond}), we have $\delta$-IOpS condition with the ultimate bound, for all $k$: 
\begin{equation}
    \|x(k)-\bar{x}(k)\|\leq \frac{\|x(0)-\bar{x}(0)\|}{(1-\bar{\rho})}+h\frac{\|\bar{B}\| \mu}{(1-\bar{\rho})},
\end{equation}
\end{enumerate}
\end{proof}

We are now ready to prove Theorem 1 and its shorter version from the \emph{main paper}.

\begin{proof}(Proof of Theorem \ref{th:stab1}: Asymptotic stability for shared weights)
\begin{enumerate}
    \item Recall that we have assumed that, for a $\tanh$ activation function,
        the network dynamics can be globally approximated by its linearisation.
        In the case of tied weights we have, for tanh activation, the steady state
        condition:
        \begin{equation}
            \bar{x}=\bar{x}+h\tanh(A\bar{x}+B{u}+b) \Leftrightarrow A\bar{x}+B{u}+b=0.
        \end{equation}
        The network has a unique input-dependant steady state $\bar{x}$ with
        steady state gain $G:$ $\|H(k)\|\rightarrow G$,  given by,
        respectively\footnote{Note that the matrix  $A$ is invertible by construction.}:
        \begin{eqnarray}\label{eq:steady_tanh}
        \bar{x}=-A^{-1}{(B{u}+b)},\\ \label{eq:steady_tanh1}
         G=\left\|\frac{\partial\bar{x}}{\partial u}\right\|=\frac{\|B\|}{\|A\|}.
        \end{eqnarray}
        From this point and the above considerations in the proof of Theorem
        \ref{th:BIBO2} we will now prove that in the case of tied weights, the
        network with $\tanh$ activation is Globally Asymptotically Stable with
        equilibrium $\bar{x}$ given by Eq. (\ref{eq:steady_tanh}). In
        order to do this we will again use the linearized system and apply
        superposition of the free response with the forced
        response. In particular, from the proof of Theorem \ref{th:BIBO2} part
        1 and from Eq. (\ref{eq:lipschits_evo}) we that the free response of
        $x(k)-\bar{x}$, for any steady state $\bar{x}$, is norm bounded by:
        $$\beta(\|e(k)\|)=\bar{\rho}^k\|x(0)-\bar{x}\|,$$
        which vanishes asymptotically. Conversely, recall that the forced
        response (of the linearised system) to the input $u$ is given by the
        discrete convolution Eq. (\ref{eq:conv}). This convolution converges for any
        chosen linearisation point as $\bar{\rho}<1$. Therefore, by linearising
        around any $\bar{x}$, we have:
        $$x(k)\rightarrow(I-J(\bar{x},u))^{-1}h\sigma^{'}(\bar{x},u)(Bu+b)=-A^{-1}(Bu+b),$$
        thus providing the desired Global Asymptotic Stability result.

    \item From the fact that the network function is Lipschitz, the network is
        also Globally Input-Output Robustly Stable, as shown for the case of
        untied weights in the proof of Theorem \ref{th:BIBO2}, part 1.
        Moreover, from Eq. (\ref{eq:lipschits_evo}) we have that the IOS property
        holds also around any nominal equilibrium point $\bar{x}$.

    \end{enumerate}
    For ReLU activations, we have that:
\begin{enumerate}
    \item The network dynamics is piece-wise linear, and
        sub-differentiable. In particular, the network has $2^n$ possible 1-step
        transitions, namely $2$ possible ones for each dimension. We will proceed
        by enumeration of all possible dynamics transition functions to show that one cannot determine the existence of a steady state or a single set of active neurons in the considered setting. This motivates the use of incremental stability. First of all, we have
        that:
    \begin{eqnarray}
        \mathcal{P}=\{(x,u) \in\mathbb{R}^{n+m}: Ax+Bu+b>0\},\\
        \mathcal{P}_i=\{(x,u) \in\mathbb{R}^{n+m}: A_{i\bullet}x+B_{i\bullet}u+b_i>0\}.
    \end{eqnarray}
    The network state at the next step is then given by:
    \begin{equation}
    x^{+}_i= \left\{
        \begin{array}{ll}
            x_i+h(A_{i\bullet}x+B_{i\bullet}u+b_{i}) & {IF}\ -A_{i\bullet}x\leq(B_{i\bullet}u+b_i), \\
            x_i & {IF}\ -A_{i\bullet}x> (B_{i\bullet}u+b_i). \\
        \end{array}\right.
    \end{equation}
    The activation slope for the $i$-th coordinate is the set valued:
    \begin{equation}
        \sigma^{'}_{ii}(\Delta x) \in \left\{
        \begin{array}{ll}
            \{1\}, & {IF}\ -A_{i\bullet}x<(B_{i\bullet}u+b_i) \\
            \{0\}, & {IF}\ -A_{i\bullet}x>(B_{i\bullet}u+b_i) \\
            {[0,1]}, & {IF}\ -A_{i\bullet}x=(B_{i\bullet}u+b_i) \\
        \end{array}\right.
    \end{equation}
    The Jacobian is again given by:
    \begin{equation}
    J(x,u)=I+\sigma^{'}(\Delta x)A.
    \end{equation}
    Clearly, if $(x(k),u)\in\mathcal{P},\ \forall k$ then the system is linear
    and time invariant. In this region the system is linear and it could admit a
    unique, input-dependant steady state, given by:
    \begin{equation}
        \bar{x}=-A^{-1}(B{u}+b),
    \end{equation}
     This, however, cannot be guaranteed as discussed next. 

    In general, we cannot expect that $(x(k),u)\in\mathcal{P},\ \forall k$. We could instead look for a  steady state to be either on the boundary
    of $\mathcal{P}$ or outside the set. In particular, if
    $(x(k),u)\not\in\mathcal{P}$ for some $k$ then $\exists\ i:
    x_i(k+t)=x_i(k),\ \forall t\geq0$ and therefore part of the network states
    will become autonomous at layer $k$, making the network simply stable if the activation stays saturated 
    (which cannot be guaranteed). Consider the case in which, at some time $k$,
    we have $(x(k),u)\not\in\mathcal{P}_i$ (for example in $\mathcal{N}_i$). In
    this case we have that $x_i(k+t)=x_i(k),\ \forall t\geq0$ and the network
    state will be free to change in its remaining dimensions with the $i$-th
    one being fixed for at least one  step. For the remaining dimensions, we can take out
    $x_i$ and consider its effect as an additional bias element. This will
    result in a smaller state space $\tilde{x} =[x_j,\dots,x_n]^T$ except for
    $x_i$, with state transfer matrix $I+h\tilde{A}$, where $\tilde{A}$
    consists of all elements of $A$, without the $i$-th row and column. Same
    for $\tilde{B}$ having the elements of $B$ except the $i$-th row. Note that
    $\tilde{A}$ still has eigenvalues in the same region as $A$ and is
    therefore negative definite and invertible. For this new system we have three
    possibilities: in the first case the trajectory stays in $\mathcal{N}_i$
    all the time and the steady states for all dimensions except $i$
    are given by:
    \begin{equation}
        \bar{\tilde{x}}=\tilde{A}^{-1}(\tilde{B}u+\tilde{b}+{\bar{J}}  \bar{x}_i)
    \end{equation}
    where $\bar{J}$ is a diagonal matrix containing the $i$-th column of $J$
    except for the $i$-th element. The other possibility is that at some point
    we also have that $(x(k),u)\not\in\mathcal{P}_j$ for some other $j$ (they 
    also can be more than one at the time) in which case one can reduce the
    state space again and repeat as above. The third possibility is,  however, that some more or even all of the activations become again unsaturated. 
    Then, the system will continue to evolve its free evolution in possibly all state dimensions, contracting once again, but changing its direction with respect to the input.
    Therefore, we can only say that the system contracts according to $\|I+A\|$ as long as the activations are not saturated. However, we cannot guarantee that the activations remain unsaturated nor that there is a steady state. Note that the network trajectory depends on the initial condition, and the fact that $(x(k),u)\in\mathcal{P}_i$ at time $k$ is also
    dependent on $x(0)$. At the same time, the contraction in $\|I+A\|$ and of all sub-matrices is sufficient to show convergence in the same norm space, for a zero input and for each pair of trajectories, as well as the bound discussed for the unshared weight case in Theorem \ref{th:BIBO2}:  $$\zeta=\frac{\|x(0)-\bar{x}(0)\|}{\bar{\rho}}.$$ 

    \item To prove Input-Output incremental practical Stability, note that each combination
        of different vector fields that make up $f(x,u)$ provides a vector
        field that is Input-Output Stable. Moreover, $f$ is uniformly continuous. We can
        therefore take the worst case IOpS gain for
        each value of the vector field to provide suitable upper bounds for the
        $\delta$-IOpS definition to be satisfied as in proof of Theorem \ref{th:BIBO2}, part 2.

\end{enumerate}
\end{proof}

Note that the gain $\gamma(\cdot)$ can be used, for instance, as a
regularizer to reduce the effect of input perturbations on the output of the network.

\section{Proof of Constraint Implementation}
In this section, the proposed implementation for fully connected
and convolutional layers is shown to be sufficient to fulfil the
stability constraint on the Jacobian spectral radius.

\subsection{Proof of Fully Connected Implementation}\label{app:fully}
Recall Algorithm 1 from the main paper. The following result is obtained:
\begin{lemma} \label{lemma:implementation}
    The Jacobian stability condition, $\rho(J(x,u))<1$, is satisfied $\forall (x,u)\in\mathcal{P}$
    for the fully connected layer if $h\leq1$ and  Algorithm 1 is used.
    \end{lemma}
Since Lemma \ref{lemma:implementation} is equivalent to Theorem 2 from the main paper,
the former will be proven next.
\begin{proof}(Proof of Lemma \ref{lemma:implementation})

Recall that Algorithm 1 results in the following operation being performed
after each training epoch:
\begin{equation}\label{eq:pr}
\tilde{R}\leftarrow \left\{
    \begin{aligned}
        & \sqrt{\delta}\frac{R}{\sqrt{\|R^TR\|_F}} & \text{IF}\ \|R^TR\|_F>\delta \\
        & {R}, & \text{Otherwise},
    \end{aligned}\right.
\end{equation}
with $\delta=1-2\epsilon \in(0,1)$.
Recall that
\begin{equation}\label{eq:par_A}
    A=-R^TR-\epsilon I.
\end{equation}
Then, Eq. (\ref{eq:pr}) is equivalent to the update:
\begin{equation}
    \label{eq:projection2}
    \tilde{A} \leftarrow \left\{
    \begin{aligned}
        & -\tilde{R}^T\tilde{R}-\epsilon I = -{(1-2\epsilon)}\frac{R^TR}{{\|R^TR\|_F}}-\epsilon I,
        & \text{IF}\ \|R^TR\|_F>1-2\epsilon \\
        & A, & \text{Otherwise}
    \end{aligned}\right.
\end{equation}
Eq. (\ref{eq:pr}) guarantees that $\|R^TR\|_F\leq (1-2\epsilon)$. From the fact
that the Frobenius norm is an upper bound of the spectral norm and because of
symmetry we have that:
\begin{equation}
    \rho(R^TR)=\|R^TR\|_2\leq\|R^TR\|_F\leq 1-2\epsilon.
\end{equation}
Recall that, from Eq.  (\ref{eq:par_A}), $A$ is negative definite and it only
has real negative real eigenvalues. Recall also Lemma \ref{lemma:shift}.
Therefore, by applying the definition in Eq. (\ref{eq:par_A})
we have that
$R=0\Rightarrow A=-\epsilon I$, then the eigenvalues of $hA$ are always located
within the interval $[-h(1-\epsilon ), -h \epsilon]$. This means that:
\begin{equation}
    \rho(hA)\leq h\max{\{\epsilon,1-\epsilon\}}.
\end{equation}

To complete the proof, recall that the network Jacobian is:
$$J(x,u)=I+h\sigma^{'}(\Delta x)A.$$
We will now look at the specific activation functions:
\begin{enumerate}
    \item For ReLU activation, we simply have that $J(x,u)=I+hA$ in the set
        $\mathcal{P}$. Lemma \ref{lemma:shift} implies that $I+hA$ has only
         positive real eigenvalues located in $[1-h(1-\epsilon),1-h\epsilon]$ which, when $h\in(0,1]$, implies that:
         \begin{equation}
         	\bar{\rho}\leq \max{\{1-h(1- \epsilon),1-h\epsilon\}}<1.
         \end{equation}

    \item For $\tanh$ activations, since the matrix
        \begin{equation}
           \bar{A} = \frac{1}{2}\left(\sigma^{'}(\Delta x)A+A^T\left(\sigma^{'}(\Delta x)\right)^T\right)
        \end{equation}
        is symmetric, $A$ is negative definite and $\sigma^{'}(\cdot)$ is diagonal
        with entries $\sigma^{'}_{ii}(\cdot)\in[\underline{\sigma},1]$ with $0<\underline{\sigma}\ll1$
         when $(x,u)\in\mathcal{P}$, then
        $\bar{A}$ is also negative definite in this set. Therefore, in virtue
        of the observations at page 399-400 of \cite{Horn:2012:MA:2422911}, we
        have that the real part of the eigenvalues of $\sigma^{'}(\Delta x)A$
        is always less than zero in $\mathcal{P}$. Namely,
        \begin{equation}\label{eq:eig}
            \text{RE}(\text{eig}(\sigma^{'}(\Delta x)A))<0.
        \end{equation}
        At the same time, by construction of $A$ and again thanks to
         $\sigma^{'}_{ii}(\cdot)\in[\underline{\sigma},1]$ and $\sigma^{'}_{ij}(\cdot)=0$ if $i\neq j$,
        we have that:
        \begin{equation}
            \begin{aligned}\label{eq:eig1}
            \rho(\sigma^{'}(\Delta x)A)\leq \|\sigma^{'}(\Delta x)A\|_2
            \leq \|\sigma^{'}(\Delta x)\|_2\|A\|_2\leq \|A\|_2 \leq 1-\epsilon
            \end{aligned}
        \end{equation}
        From the above considerations the real part of the eigenvalues of $h\sigma^{'}(\Delta x)A$ is in the interval
        $[-h(1-\epsilon ), -h \underline{\sigma} \epsilon]$.
        Assume $h\leq1$.

        Finally, we show that the Jacobian has only positive real eigenvalues.
        From Theorem 2.2 in \cite{semidef_prod}, we know that $\sigma^{'}(\Delta x)A$ is \emph{similar} to:
        \begin{eqnarray}
        		\sigma^{'}(\Delta x)^{-1/2}\sigma^{'}(\Delta x)A\sigma^{'}(\Delta x)^{1/2} = \sigma^{'}(\Delta x)^{1/2}A\sigma^{'}(\Delta x)^{1/2} \\
		=\sigma^{'}(\Delta x)^{1/2}A{(\sigma^{'}(\Delta x)^{1/2})}^T
        \end{eqnarray}
        which is symmetric, negative definite and therefore has only negative real eigenvalues, provided that $(x,u)\in\mathcal{P}$.
         Since similarity implies eigenvalue equivalence, then $\sigma^{'}(\Delta x)A$
         has only negative real eigenvalues.
          Then, from Lemma \ref{lemma:shift}, and from Eq. (\ref{eq:eig}) and Eq.
        (\ref{eq:eig1}) we have that the eigenvalues of
        $J(x,u)$ are positive real and contained in $[1-h(1-\epsilon ), 1-h \underline{\sigma} \epsilon]$.
        Therefore we have that
        $\bar{\rho}\leq \max{\{1-h(1-\epsilon ),1-h \underline{\sigma} \epsilon\}}$,
        which is less than $1$ as $\underline{\sigma}>0$ by definition.
\end{enumerate}
\end{proof}
The less restrictive bound $\delta=2(1-\epsilon)$ with $h\leq1$ is also sufficient for
 stability but it can result in trajectories that oscillate since it it does not constrain the eigenvalues to be positive real.
  Practically speaking the bound $\delta=2(1-\epsilon)$ has has proven
sufficient for our MNIST experiments to be successful, however,  we believe it is important to stress the difference between
this and our proposed bound. In particular, our solution for fully connected layers
leads to a critically damped system, i.e. to a monotonic trajectory for the 1D case.
This means that we can expect the activations to behave monotonically both in time and in space.
This behaviour is demonstrated in Section \ref{sec:monotonicity}.
The additional regularity of the resulting function acts as a stronger regularisation on the network.

Note also that in the above proof we have shown that, in the case
of ReLU, the $2$-norm is suitable to prove stability. Moreover, if $h=1$ and
 $\epsilon\leq 0.5$ then we have $\bar{\rho}=1-\epsilon$.

\subsection{Derivation of Convolutional Layer Implementation}\label{app:CNN}

\subsubsection{Mathematical derivation of the proposed algorithm}
Denote the convolution operator in NAIS-Net,
 for each latent map $c$, as the following:
\begin{equation} \label{eq:CNNmodel_rewritten2_sup}
   \begin{aligned}
       X^{(c)}(k+1) =X^{(c)}(k)+h\sigma\left(\Delta X^{(c)}(k)\right), \\
   \end{aligned}
\end{equation}
where:
\begin{equation}
\Delta X^{(c)}(k) = \sum_i C^{(c)}_{(i)}*X^{(i)}(k)
         + \sum_j D^{(c)}_{(j)}*U^{(j)}+E^{(c)},
\end{equation}
and where $X^{(i)}(k)$, is the layer state matrix for channel $i$, $U^{(j)}$,
is the layer input data matrix for channel $j$ (where an appropriate zero
padding has been applied) at layer $k$.
The activation function, $\sigma$, is again applied element-wise.
Recall also Algorithm 2 from the main paper.
The following Lemma is obtained (equivalent to Theorem 3 from the main paper):
\begin{lemma} \label{lemma:CNN}

  If Algorithm 2 is used, then the Jacobian stability condition, $\rho(J(x,u))<1$,
  is satisfied $\forall (x,u)\in\mathcal{P}$ for the convolutional layer with $h\leq1$.
\end{lemma}
The first step to obtain the result is to prove that the convolutional layer can
 be expressed as a suitable fully connected layer as in Lemma 1 from the main
 paper. This is shown next.

\begin{proof}(Proof of Lemma 1 from the main paper)

For layer $k$, define $X(k)$ as a tall matrix containing all the state channel
matrices $X^{(c)}(k)$ stacked vertically, namely:
\begin{equation}
X = \left(\begin{array}{c}
	X^{(1)}\\ X^{(2)} \\ \vdots\\ X^{(N_c)}
	\end{array}\right).
\end{equation}
Similar considerations apply to the input matrix $U$
and the bias matrix $E$. Then, convolutional layers can
be analysed in a similar way to fully connected layers, by means of the
following vectorised representation:
\begin{equation}\label{eq:vectorisedCNN}
    \begin{aligned}
        x(k+1) = x(k)+h\sigma\big(Ax(k)+Bu+b\big),
    \end{aligned}
\end{equation}
where $x(k)\in\mathbb{R}^{n_X^2 \cdot N_c}$,
$u=\mathbb{R}^{n_U^2\cdot N_c}$, where $N_c$ is
the number of channels, $n_X$ is the size of the latent space matrix for a single
channel, while $n_U$ is the size of the input data matrix for a single channel.
In eq. (\ref{eq:vectorisedCNN}), the matrices $A$ and $B$ are made of blocks
containing the convolution filters elements in a particular structure, to be
characterised next. First, in order to preserve dimensionality of $x$, the
convolution for $x$ will have a fixed stride of $1$, a filter size $n_C$ and a
zero padding of $p\in\mathbb{N}$, such that $n_C=2p+1$. If a greater stride is
used, then the state space can be extended with an appropriate number of
constant zero entries (not connected). Let's then consider, without loss of
generality, a unitary stride. The matrix $A$ is all we need to define
in order to prove the Lemma.

In eq. (\ref{eq:vectorisedCNN}), the vector $x$ is chosen to be the vectorised
version of $X$. In particular,
\begin{equation}\label{eq:smallx}
	x=\left(\begin{array}{c}
	x^{(1)}\\ x^{(2)} \\ \vdots\\ x^{(N_c)}
	\end{array}\right),
\end{equation}
where $x^{(c)}$ is the vectorised version of $X^{(c)}$. More specifically,
these objects are defined as:
\begin{equation}
	X^{(c)} = \left(
        \begin{array}{cccc}
            X^{(c)}_{1,1} & X^{(c)}_{1,2} & \cdots & X^{(c)}_{1, n} \\
            X^{(c)}_{2,1} & X^{(c)}_{2,2} & \cdots & X^{(c)}_{2, n} \\
            \vdots & \vdots & \ddots & \vdots \\
            X^{(c)}_{n,1} & X^{(c)}_{n, 2}  & \cdots & X^{(c)}_{n, n}
        \end{array}
    \right) \in \mathbb{R}^{n_X\times n_X},
\end{equation}
and
\begin{equation}
	x^{(c)}=\left(\begin{array}{c}
	X^{(c)}_{1,1} \\ X^{(c)}_{1,2} \\ \vdots\\ X^{(c)}_{n,n}
	\end{array}\right)\in\mathbb{R}^{n_X^2}.
\end{equation}
Similar considerations apply to $u$.
The matrix $A$ in Eq. (\ref{eq:vectorisedCNN})
has the following structure:
\begin{equation}
    A =
    \left(
        \begin{array}{cccc}
          A^{(1)}_{(1)} & A^{(1)}_{(2)}  & \dots & A^{(1)}_{(Nc)} \\
          \\
          A^{(2)}_{(1)} & A^{(2)}_{(2)}  & \dots & A^{(2)}_{(Nc)} \\
          \vdots &  \vdots   &   \ddots &  \vdots   \\
            A^{(Nc)}_{(1)} & A^{(Nc)}_{(2)}  & \dots & A^{(Nc)}_{(Nc)} \\
        \end{array}
    \right)\in\mathbb{R}^{(n^2 \cdot N_c) \times (n^2 \cdot N_c)},
\end{equation}
where $Nc$ is the number of channels and $A^{(c)}_{(i)}$ corresponds to the
filter $C^{(c)}_{(i)}$. In particular, each row of ${A}$ contains in fact the
elements of the filter $C^{(c)}_{(i)}$, plus some zero elements, with the
central element of the filter $C^{(c)}_{(i)}$ on the diagonal. The latter point
is instrumental to the proof and can be demonstrated as follows. Define the
single channel filters as:
\begin{equation}
    C^{(c)}_{(i)} = \left(
        \begin{array}{cccc}
        {C^{(c)}_{(i)}}_{1,1} & {C^{(c)}_{(i)}}_{1,2} & \cdots & {C^{(c)}_{(i)}}_{1, n_C} \\
        {C^{(c)}_{(i)}}_{2,1} & {C^{(c)}_{(i)}}_{2,2} & \cdots & {C^{(c)}_{(i)}}_{2, n_C} \\
        \vdots & \vdots & \ddots & \vdots \\
        {C^{(c)}_{(i)}}_{n_C,1} & {C^{(c)}_{(i)}}_{n_C, 2} & \cdots & {C^{(c)}_{(i)}}_{n_C, n_C}
        \end{array}
    \right).
\end{equation}
Consider now the output of the single channel convolution
$Z^{(c)}_{(i)}=C^{(c)}_{(i)}*X^{(i)}$ with the discussed padding and stride.
The first element of the resulting matrix, ${Z^{(c)}_{(i)}}_{1,1}$ is
determined by applying the filter to the first patch of $X^{(i)}$, suitably
padded. For instance, for a $3$-by-$3$ filter ($p=1$), we have:
\begin{equation}
	\text{patch}_{1,1}\left({X^{(i)}}\right) =  \left(\begin{array}{cccccc}
    \tikzmark{left}{0} & 0 & 0 & 0 & \cdots & 0 \\
	0 &  X^{(c)}_{1,1} & X^{(c)}_{1,2}  & \cdots & X^{(c)}_{1, n} & 0 \\
    0 & X^{(c)}_{2,1} & \tikzmark{right}{X^{(c)}_{2,2}}  & \cdots & X^{(c)}_{2, n}  & 0\\
	\vdots &   \vdots & \vdots & \ddots & \vdots & \vdots \\
	   0& X^{(c)}_{n,1} & X^{(c)}_{n, 2}  & \cdots & X^{(c)}_{n, n} & 0
	   \end{array}\right).
\end{equation}
\Highlight
The first element of ${Z^{(c)}_{(i)}}$ is therefore given by:
\begin{equation}
	{Z^{(c)}_{(i)}}_{1,1}= X^{(i)}_{1,1}\ {C^{(c)}_{(i)}}_{i_\text{centre},
        i_\text{centre}} + X^{(i)}_{1,2}\ {C^{(c)}_{(i)}}_{i_\text{centre},
        i_\text{centre}+1} + \dots + X^{(i)}_{n_C-p,n_C-p}\ {C^{(c)}_{(i)}}_{n_C, n_C},
\end{equation}
where $i_\text{centre}$ denotes the central row (and column) of the filter.

The second element of the first row can be computed by means of the following
patch (again when $p=1$, for illustration):
\begin{equation}
	\text{patch}_{1,2}\left({X^{(i)}}\right) =  \left(\begin{array}{cccccc}
    0 & \tikzmark{left}{0} & 0 & 0 & \cdots & 0 \\
	0 &  X^{(c)}_{1,1} & X^{(c)}_{1,2}  & \cdots & X^{(c)}_{1, n} & 0 \\
    0 & X^{(c)}_{2,1} & X^{(c)}_{2,2}  & \tikzmark{right}{\cdots} & X^{(c)}_{2, n}  & 0\\
	\vdots &   \vdots & \vdots & \ddots & \vdots & \vdots \\
	   0& X^{(c)}_{n,1} & X^{(c)}_{n, 2}  & \cdots & X^{(c)}_{n, n} & 0
	   \end{array}\right).
\end{equation}
\Highlight
The element is therefore given by:
\begin{equation}
	{Z^{(c)}_{(i)}}_{1,2}= X^{(i)}_{1,1}\ {C^{(c)}_{(i)}}_{i_\text{centre},
        i_\text{centre}-1} + X^{(i)}_{1,2}\ {C^{(c)}_{(i)}}_{i_\text{centre},
        i_\text{centre}} + \dots + X^{(i)}_{n_C-p,n_C}\ {C^{(c)}_{(i)}}_{n_C, n_C}.
\end{equation}
The remaining elements of ${Z^{(c)}_{(i)}}$ will follow a similar rule, which
will involve (in the worst case) all of the elements of the filter.
In particular, one can notice for $Z_{ij}$ the corresponding element of $X$,
$X_{ij}$, is always multiplied by
${C^{(c)}_{(i)}}_{i_\text{centre},  i_\text{centre}}$.
In order to produce the matrix $A$, we can consider the Jacobian of $Z$. In particular, for the first two rows of ${A^{(c)}_{(i)}}$
we have:
\begin{equation} \label{eq:A1}
    {A^{(c)}_{(i)}}_{1,\bullet} = \frac{\partial {Z^{(c)}_{(i)}}_{1,1}}{\partial  X^{(i)}}
	= \left[
        \begin{array}{ccccc}
        {C^{(c)}_{(i)}}_{i_\text{centre},  i_\text{centre}} &
        {C^{(c)}_{(i)}}_{i_\text{centre}+1, i_\text{centre}+1}& \dots& {C^{(c)}_{(i)}}_{n_C, n_C}
		& 0
	    \end{array}
      \right],
\end{equation}
where $[\cdot]$ is used to define a row vector, and where
${A^{(c)}_{(i)}}_{1,\bullet}$ has an appropriate number of zeros at the end,
and
\begin{equation} \label{eq:A2}
    {A^{(c)}_{(i)}}_{2,\bullet} = \frac{\partial{Z^{(c)}_{(i)}}_{1,2}}{\partial  X^{(i)}}
    =  \left[
        \begin{array}{ccccc}
            {C^{(c)}_{(i)}}_{i_\text{centre},  i_\text{centre}-1} &
            {C^{(c)}_{(i)}}_{i_\text{centre},  	i_\text{centre}} &
            \dots &   {C^{(c)}_{(i)}}_{n_C, n_C}
            & 0
        \end{array}\right].
\end{equation}
Note that, the vectors defined in eq. (\ref{eq:A1}) and eq. (\ref{eq:A2}),
contain several zeros among the non-zero elements.
By applying the filter $C_{(i)}^{(c)}$ to the remaining patches of $X^{(i)}$
 one can inductively construct
the matrix $A_{(i)}^{(c)}$. It can also be can noticed that each
row ${A^{(c)}_{(i)}}_{j,\bullet}$ contains \emph{at most} all of the elements
of the filter, with the central element of the filter in position $j$. By
stacking together the obtained rows ${A^{(c)}_{(i)}}_{j,\bullet}$ we obtain a
matrix, ${A^{(c)}_{(i)}}$, which has ${C^{(c)}_{(i)}}_{i_\text{centre},
i_\text{centre}}$ on the diagonal. Each row of this matrix contains, in the
worst case, all of the elements of ${C^{(c)}_{(i)}}$.

Define the vectorised version of ${Z^{(c)}_{(i)}}$ as:
\begin{equation}
	z^{(c)}_{(i)} =
     \left(
        \begin{array}{c}
	        {Z^{(c)}_{(i)}}_{1,1} \\
            {Z^{(c)}_{(i)}}_{1,2} \\
            \vdots\\
            {Z^{(c)}_{(i)}}_{n,n}
	    \end{array}
     \right) \in \mathbb{R}^{n^2},
\end{equation}
which, by linearity, satisfies:
\begin{equation}
	z^{(c)}_{(i)} = {A^{(c)}_{(i)}}x^{(i)}.
\end{equation}
By summing over the index $i$ we obtain the vectorised output of the
convolution $C*X$ for the channel $c$:
\begin{equation}\label{eq:zchan}
	z^{(c)} = \sum_{i=1}^{N_c}{A^{(c)}_{(i)}}x^{(i)}=
	\left[\begin{array}{cccc}
		{A}^{(c)}_{(1)} & {A}^{(c)}_{(2)}  & \dots & {A}^{(c)}_{(N_c)}
	\end{array}
	\right] x=A^{(c)}x,
\end{equation}
where the matrices ${A}^{(c)}_{(i)}$ are stacked horizontally and where we have
used the definition of $x$ given in eq. (\ref{eq:smallx}). The full matrix $A$
is therefore given by:
\begin{equation}\label{eq:Alines}
{A}=\left(
\begin{array}{c}
  {A}^{(1)} \\
  {A}^{(2)} \\
  \vdots  \\
    {A}^{(N_c)} \\
\end{array}
\right)=
\left(
\begin{array}{cccc}
  {A}^{(1)}_{(1)} & {A}^{(1)}_{(2)}  & \dots & {A}^{(1)}_{(N_c)} \\ \hline
  {A}^{(2)}_{(1)} & {A}^{(2)}_{(2)}  & \dots & {A}^{(2)}_{(N_c)} \\ \hline
  \vdots &  \vdots   &   \ddots &  \vdots   \\ \hline
    {A}^{(N_c)}_{(1)} & {A}^{(N_c)}_{(2)}  & \dots & {A}^{(N_c)}_{(N_c)} \\
\end{array}
\right),
\end{equation}
where we have conveniently separated the long blocks used to produce the single channels result,
 $z^{(c)}$.  By defining the vector
 \begin{equation}\label{eq:smallz}
	z=\left(\begin{array}{c}
	z^{(1)}\\ z^{(2)} \\ \vdots\\ z^{(N_c)}
	\end{array}\right),
\end{equation}
and by means of eq. (\ref{eq:zchan}) we have that $z=Ax$. This is a vectorised
representation of $C*X$.
\end{proof}
We are now ready to prove Lemma \ref{lemma:CNN} and consequently
Theorem 3 from the main paper.
\begin{proof}(Proof of Lemma \ref{lemma:CNN})

Similar to what done for the fully connected layer, we will first show that the
Algorithm 2 places the eigenvalues of $I+A$ strictly inside the unit circle.
Then, we will also show that $J(x,u)$ enjoys the same property in
$\mathcal{P}$.

We will now derive the steps used in Algorithm 2 to enforce that
$\rho(I+A)\leq1-\epsilon$. Recall that, from Lemma \ref{lemma:shift}, we need
the eigenvalues $A$ to be lying inside the circle, $\mathcal{S}$, of the
complex plane centered at $(-1, 0\jmath)$ with radius $1-\epsilon$. More formally:
\begin{equation}
	\mathcal{S}=\{\lambda \in\mathbb{C}: |\lambda + 1|\leq 1-\epsilon\}.
\end{equation}
The Gershgorin theorem offers a way to locate the eigenvalues of ${A}$ inside
$\mathcal{S}$.  Consider the $c$-th long long block in eq. (\ref{eq:Alines}),
$A^{(c)}$.  Recall the particular structure of  ${A}_{(c)}^{(c)}$ as
highlighted in eq. (\ref{eq:A1}) and eq. (\ref{eq:A2}). For each long block
$A^{(c)}$ we have that all associated Gershgorin disks must satisfy:
\begin{equation}\label{eq:Gershgorin}
  \left|\lambda-C_{i_\text{centre}}^{(c)}\right|\leq\sum_{i\neq{i_\text{centre}}}\left|C_i^{(c)}\right|,
\end{equation}
where $\lambda$ is any of the corresponding eigenvalues,
$C_{i_\text{centre}}^{(c)}$ is the \emph{central element} of the filter
${C}_{(c)}^{(c)}$, namely, $C_{i_\text{centre}}^{(c)}={C^{(c)}_{(c)}}_{i_\text{centre}, i_\text{centre}}$,
and the sum is performed over all of the remaining elements of ${C}_{(c)}^{(c)}$
plus all the elements of the remaining filters used for,
$z^{(c)}$, namely ${C}^{(c)}_{(j)}$, $\forall j\neq c$. Therefore, one can act
directly on the kernel $C$ without computing $A$.

By means of Algorithm 2, for each channel $c$ we have that:
\begin{equation}
	C_{i_\text{centre}}^{(c)} = -1 -\delta_c,
\end{equation}
and
\begin{equation}
	\sum_{i\neq{i_\text{centre}}}\left|C_i^{(c)}\right|\leq 1-\epsilon -|\delta_c|
\end{equation}
where
\begin{equation}\label{eq:hyper}
|\delta_c|<1-\eta,\ 0<\epsilon<\eta<1.
\end{equation}
This means that for every $c$, the $c$-th block  of $A$ has the largest
Gershgorin region bounded by a disk, $\mathcal{S}^c$. This disk is centred at
$(-1-\delta_c, 0\jmath)$ with radius $1-\epsilon-|\delta_c|$. More formally, the
disk is defined as:
 \begin{equation}
	\mathcal{S}^c=\{\lambda \in\mathbb{C}: |\lambda + 1+\delta_c|\leq 1-\epsilon-|\delta_c|\}.
\end{equation}
Thanks to eq. (\ref{eq:hyper}), this region is not empty and its center can be
only strictly inside $\mathcal{S}$. Clearly, we have that,
$\mathcal{S}^c\subseteq\mathcal{S}, \forall c$. This follows from convexity and
from the fact that, when $|\delta_c|=1-\eta$, thanks to eq. (\ref{eq:hyper}) we have
\begin{equation}\label{eq:gerhdiskdelta}
    \mathcal{S}^c=\{\lambda \in\mathbb{C}: |\lambda +1\pm (1-\eta)|\leq 1-\epsilon-(1-\eta)\}\subset\mathcal{S}
\end{equation}
while on the other hand
$\delta_c=0 \Rightarrow \mathcal{S}^c=\mathcal{S}$.

We will now show that the eigenvalue condition also applies to the network
state Jacobian. From Algorithm 2 we have that each row $j$ of the matrix
$\underbar{J}=I+A$ satisfies (for the corresponding channel $c$):
\begin{equation}
    \begin{aligned}
        \sum_i \left|\underbar{J}_{ji}\right| &\leq \left|1+C_{i_\text{centre}}^{(c)}\right|
        + \sum_{i\neq{i_\text{centre}}}\left|C_i^{(c)}\right|\\
                &\leq |\delta_c| + 1-\epsilon - |\delta_c|=1-\epsilon.
    \end{aligned}
\end{equation}
Thus $\|\underbar{J}\|_\infty=\|I+A\|_\infty\leq 1-\epsilon<1$.
\footnote{Note that this also implies that $\rho(I+A)\leq1-\epsilon$, by means
    of the matrix norm identity $\rho(\cdot)\leq\|\cdot\|$. This was however
    already verified by the Gershgorin Theorem.}
Recall that, in the vectorised representation, the state Jacobian is:
 $$J(x,u)=I+h\sigma^{'}(\Delta x)A.$$
 Then, if $h\leq1$,  we have that $\forall (x,u)\in\mathcal{P}$:
\begin{equation}
    \begin{aligned}
        \|J(x,u)\|_\infty &= \max_i |1-h\sigma^{'}_{ii}(\Delta x) (1+\delta_c)| + h\sigma^{'}_{ii}(\Delta x)\sum_{j\neq i} |A_{ij}|\\
        &\leq\max_i  \{1-h\sigma^{'}_{ii}(\Delta x) +h\sigma^{'}_{ii}(\Delta x)|\delta_c| + h\sigma^{'}_{ii}(\Delta x)(1-\epsilon)-h\sigma^{'}_{ii}(\Delta x)|\delta_c|\} \\
        &= \max_i \{1-h\sigma^{'}_{ii}(\Delta x)\epsilon\}<1-h\underline{\sigma}\epsilon\leq1.
    \end{aligned}
\end{equation}
The proof is concluded by means of the identity $\rho(\cdot)\leq\|\cdot\|$ .
\end{proof}
Note that in the above we have also showed that in the convolutional case the infinity norm
is suitable to prove stability.
Moreover, for ReLU we have that $\bar{\rho}\leq 1-h\epsilon$.

The above results can directly be extended to the untied weight case providing
that the projections are applied at each stage of the unroll.

In the main paper we have used $\delta_c=0,\ \forall c$ (equivalently, $\eta=1$).
This means that one filter weight per latent channel is fixed at $-1$ which
 might seem conservative. It is however worth noting that this choice provides the biggest Gershgorin disk
  for the matrix $A$ as defined in Eq. (\ref{eq:gerhdiskdelta}). Hence $\delta_c=0,\ \forall c$ results in the least
   restriction for the remaining elements of the filter bank.
   At the same time, we have $N_C$ less parameters to train.
    A tradeoff is however
   possible if one wants to experiment with different values of $\eta\in(\epsilon,1)$.

\subsubsection{Illustrative Example}
Consider, for instance, $3$ latent channels $X^{(1)},\dots,X^{(3)}$, which
results in $3$ blocks of $3$ filters, $1$ block per channel $X^{(c)}(k+1)$.
Each filter has the same size $n_C$, which again needs to be chosen such that
the channel dimensionality is preserved. This corresponds to the following
matrix:
\begin{equation}\label{eq:toeplitz_explained}
    \begin{aligned}
    {A} & =
    \left(
    \begin{array}{c|c|c}
      {A}^{(1)}_{(1)} & {A}^{(1)}_{(2)} & {A}^{(1)}_{(3)} \\ \hline
      {A}^{(2)}_{(1)} & {A}^{(2)}_{(2)} & {A}^{(2)}_{(3)} \\ \hline
      {A}^{(3)}_{(1)} & {A}^{(3)}_{(2)} & {A}^{(3)}_{(3)} \\
    \end{array}
    \right) = \\ & =
    \left(
    \begin{array}{cc|cc|cc}
        C_{i_\text{centre}}^{(1)} &  C_j^{(1)} & C_{n_C-1}^{(1)}  & \dots & \dots   & C_{3 \cdot n_C-1}^{(1)} \\
        \vdots & \vdots  &  \ddots & \dots  & \dots & \vdots \\ \hline
        \tilde C_1^{(2)} &  \dots  & C_{i_\text{centre}}^{(2)} & \dots & \dots & C_{3 \cdot c_C-1}^{(2)} \\
        \vdots & \vdots  &  \ddots & \dots & \dots & \vdots \\ \hline
        \tilde C_1^{(3)} &  \dots  & \dots & \dots & C_{i_\text{centre}}^{(3)} & \vdots  \\
        \vdots & \vdots  &  \ddots & \dots & \dots & \ddots \\
    \end{array}
    \right),
    \end{aligned}
\end{equation}
where the relevant rows include all elements of the filters (in the worst case)
together with a large number of zeros, the position of which does not matter for our
results, and the diagonal elements are known
and non-zero. Therefore, according to eq. (\ref{eq:toeplitz_explained}), we have
to repeat Algorithm 2 from \emph{main paper} for each of
the $3$ filter banks.

\section{Analysis of 1D Activations}\label{sec:monotonicity}
This paper has discussed the formulation and implementation of stability conditions for a non-autonomous dynamical system, inspired by the ResNet, with the addition of a necessary input skip connection.
This system's unroll, given a constant input, results in a novel architecture for supervised learning.
In particular, we are interested in using the system's trajectory in two ways. The first consists in using the last point of an unroll of fixed length $K$. The second one uses an unroll of varying length less or equal to $K>0$, which depends on a stopping criterion $\|\Delta x(x,u)\|<\epsilon$, with $\epsilon>0$.
In the former case, this produces an input-output map (a NAIS-Net block) that is Lipschitz for any $K$ (even infinity) and significantly better behaved than a standard ResNet without the use regularisation or batch normalisation. The latter case, defines a \emph{pattern-dependent processing depth} where the network results instead in a piecewise Lipschitz function for any $K>0$ and $\epsilon>0$.

In order to clarify further the role of stability, the behaviour of a scalar NAIS-Net block and its unstable version, \emph{non-autonomous ResNet}, are compared. Recall that the NAIS-Net block is defined by the unroll of the dynamical system:
\begin{equation}\label{eq:DNN_tied_again}
    x(k+1) = x(k) + \Delta x(x(k),u(k)) = x(k)+h\sigma\bigg(Ax(k)+Bu+b\bigg).
\end{equation}
In this Section, since the scalar case is considered, $A$, $B$ and $b$ are scalar. In particular, we investigate different values of $A$ that satisfy or violate our stability conditions, namely $\|I+A\|<1$, which in this special case translate into $A \in (-2, 0)$. Whenever this condition is violated we denote the resulting network as \emph{non-autonomous ResNet} or, equivalently, \emph{unstable NAIS-Net}.
The input parameter $B$ is set to $1$ and the bias $b$ is set to zero for all experiments. Note that, in this setting, varying the input is equivalent to varying the bias.
\subsection{Fixed number of unroll steps}
Figure \ref{sec:1d_unstable_1} shows a couple of pathological cases in which the unstable non-autonomous ResNet produces unreliable or uninformative input-output maps. In particular, the top graphs present the case of positive unstable eigenvalues. The top left figure shows that in this case the tanh activated network presents bifurcations which can make gradients explode \cite{pascanu2013a}. The slope is nearly infinite  around the origin and finally, for large inputs, the activation collapses into a flat function equal to $k A$, where $k$ is the iteration number. The ReLU network (top right) has an exponential gain increase per step $k$ and the gain for $k=100$ reaches $10^{30}$ (see red box and pointer). Bottom graphs present the case of large negative eigenvalues. In the bottom left figure, the tanh activation produces a map that has limited slope but is also quite irregular, especially around the origin. This can also result into large gradients during training. The ReLU activated network (bottom right) instead produces an uninformative map, $\max(u,0)$, which is (locally) independent from the parameter $A$ and the unroll length $K$.

\begin{figure}[t!]
\centering
    \noindent
    \begin{minipage}{0.48\columnwidth}
    \begin{figure}[H]
        \centering
            \includegraphics[width=\linewidth, clip]{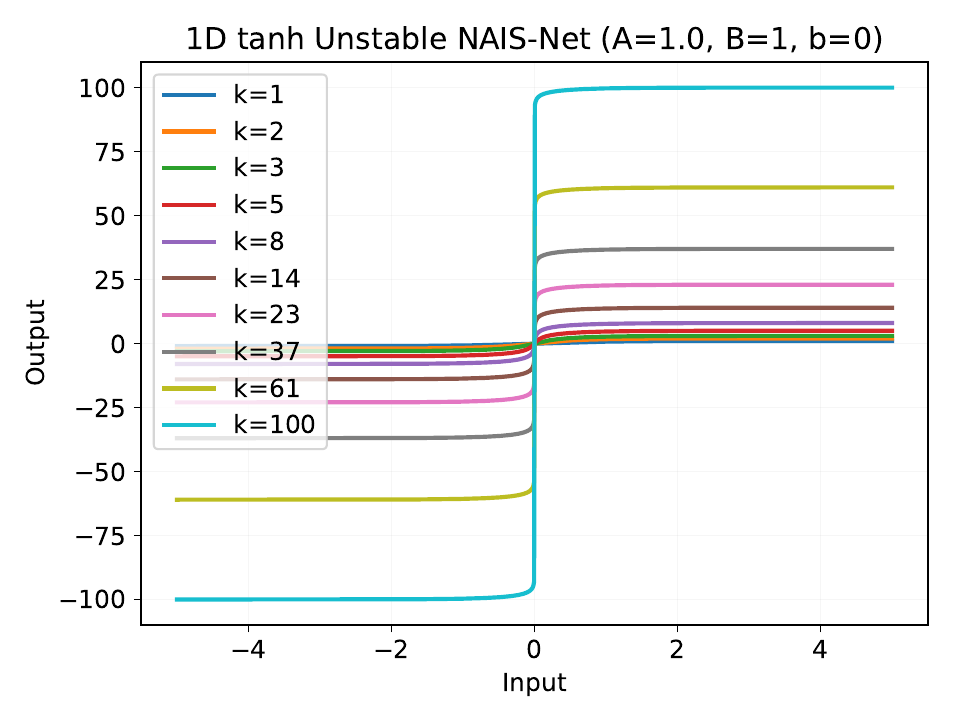}
    \end{figure}
    \end{minipage}
    \noindent
    \begin{minipage}{0.48\columnwidth}
    \begin{figure}[H]
     \includegraphics[width=\linewidth, trim={0cm, 7cm, 0cm, 7cm}, clip]{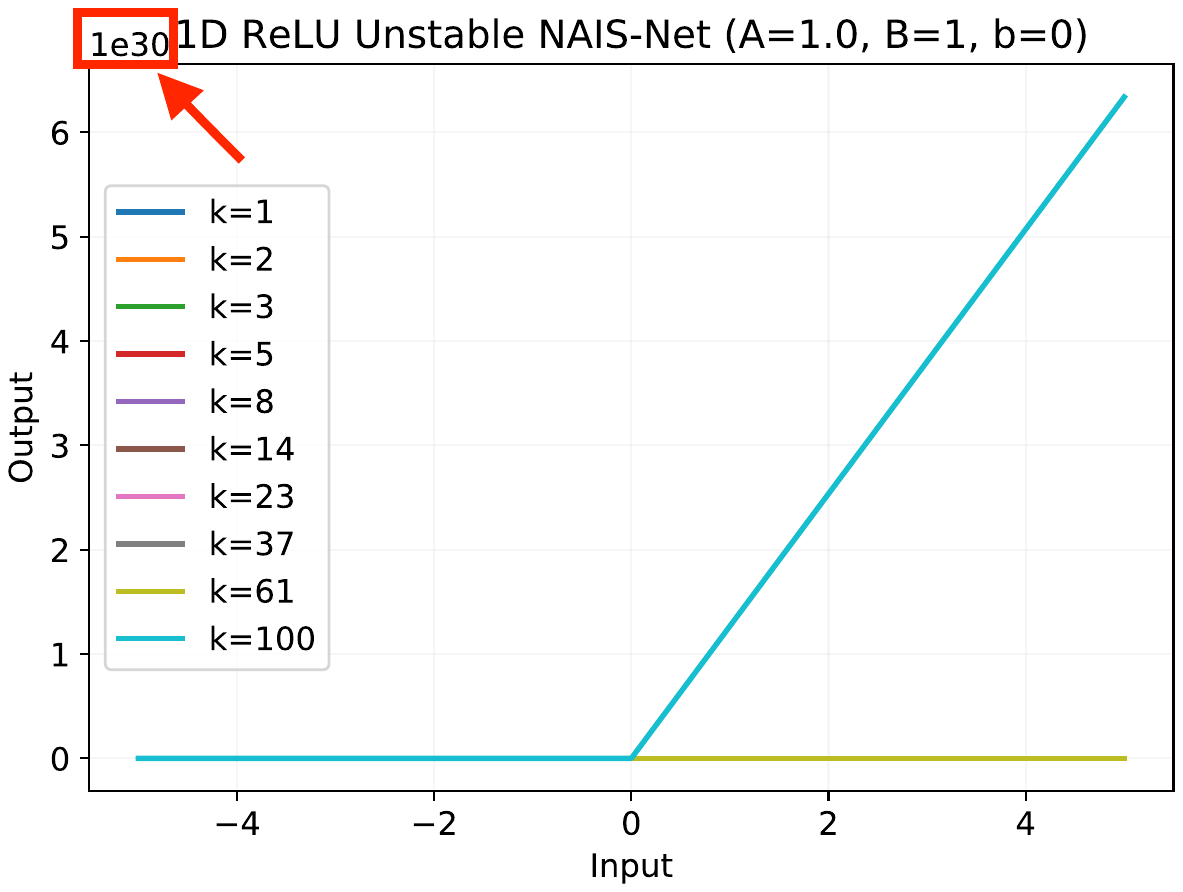}
     \end{figure}
    \end{minipage}
     \begin{minipage}{0.48\columnwidth}
    \begin{figure}[H]
        \centering
            \includegraphics[width=\linewidth, clip]{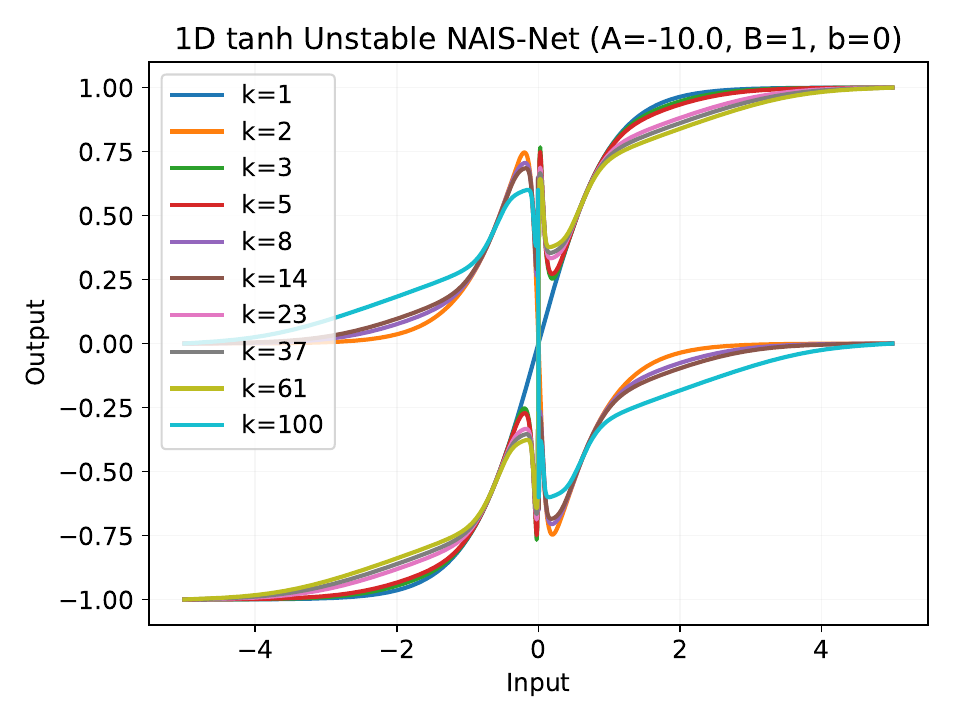}
    \end{figure}
    \end{minipage}
    \noindent
    \begin{minipage}{0.48\columnwidth}
    \begin{figure}[H]
     \includegraphics[width=\linewidth]{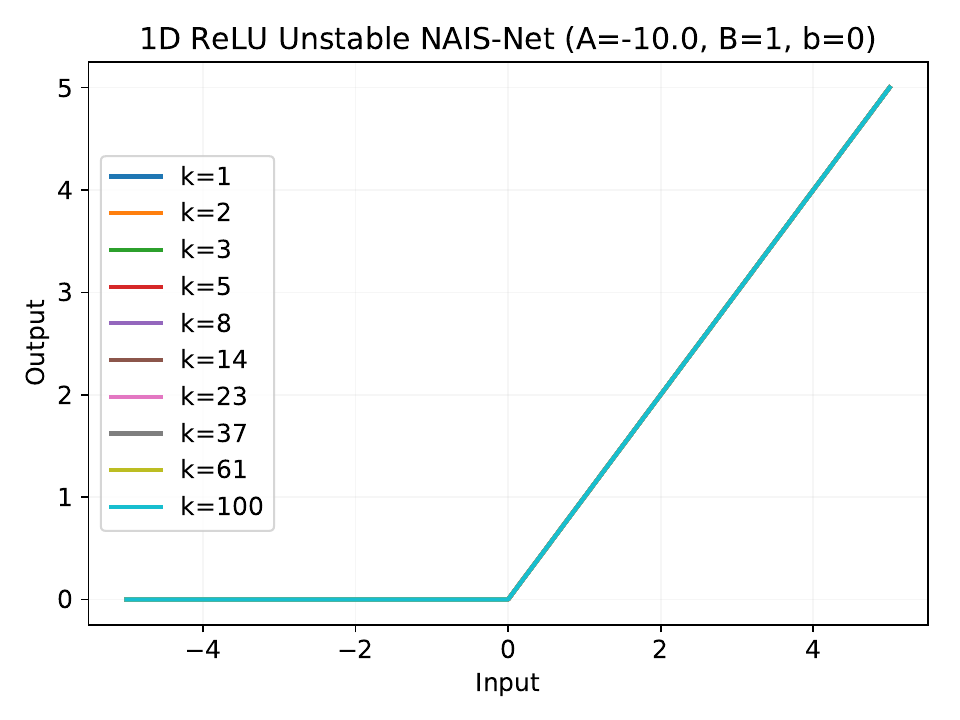}
     \end{figure}
    \end{minipage}

    \caption{
    \small{\bf Single neuron unstable NAIS-Net (non-autonomous ResNet). Input-output map for different unroll length for tanh (Left) and ReLU (Right) activations.}
    Pathological cases in which the unstable non-autonomous ResNet produces unreliable or uninformative input-output maps. In particular, the top graphs present the case of positive unstable eigenvalues. The top left figure shows that in this case the tanh activated network presents bifurcations which can make gradients explode \cite{pascanu2013a}. The slope is also nearly infinite  around the origin. Elsewhere, the activation collapses into a flat function equal to $k A$, where $k$ is the iteration number. The ReLU activations (top right) has an exponential gain increase per step $k$ and the gain for $k=100$ reaches $10^{30}$ (see red box and pointer). Bottom graphs present the case of large negative eigenvalues. In the bottom left figure, the tanh activation produces a map that has limited slope but it is quite irregular, especially around the origin. This can also result large gradients during training. The ReLU activated network (bottom right) instead produces an uninformative map, $\max(u,0)$, locally independent from $A$ and the unroll length $K$.
    }
    \label{sec:1d_unstable_1}
\end{figure}

Figure \ref{sec:1d_stable_1} shows the input-output maps produced by stable NAIS-Net with our proposed reprojection for fully connected architectures. In particular, the top and bottom graphs present the case of positive real stable eigenvalues with different magnitude. This means that the resulting trajectories are critically damped, namely, oscillation-free. The left figures shows that as a result the stable tanh activated networks have monotonic activations (strictly increasing around the origin) that tend to a straight line as $k\rightarrow\infty$. This is confirmed by the theoretical results. Moreover, the map is Lipschitz with Lipschitz constant equal to the steady state gain presented in the main paper, $\|A^{-1}\|\ \|B\|$. Figures on the right present the ReLU case where the maps remain of the same form but change in slope until the theoretical gain $\|A^{-1}\|\ \|B\|$ is reached. This means that one cannot have an unbounded slope or the same map for different unroll length. The rate of change of the map as a function of $k$ is also determined by the parameters but it is always under control.
Although this was not proven in our results, the behaviour of the resulting input-output maps suggests that gradient should be well behaved during training and not explode nor vanish.  Finally, one could make the Lipschitz constant unitary by multiplying the network output by $\|A|\//\|B\|$ once the recursion is finished. This could be used as an alternative to batch normalization and it will be investigated in future.

\begin{figure}[t!]
\centering
    \noindent
    \begin{minipage}{0.48\columnwidth}
    \begin{figure}[H]
        \centering
            \includegraphics[width=\linewidth, clip]{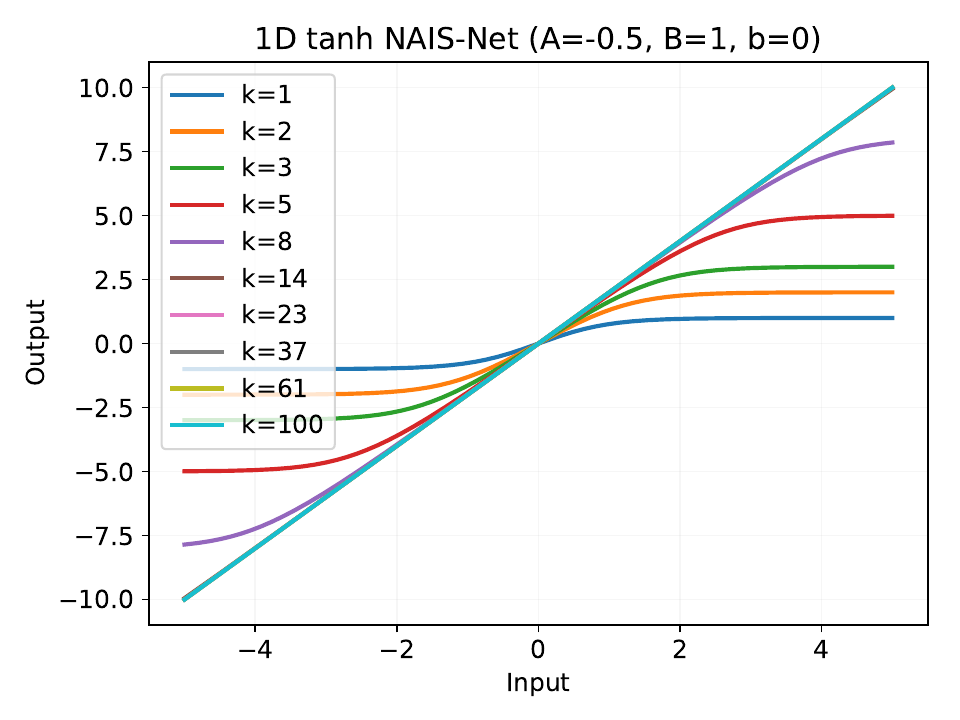}
    \end{figure}
    \end{minipage}
    \noindent
    \begin{minipage}{0.48\columnwidth}
    \begin{figure}[H]
     \includegraphics[width=\linewidth]{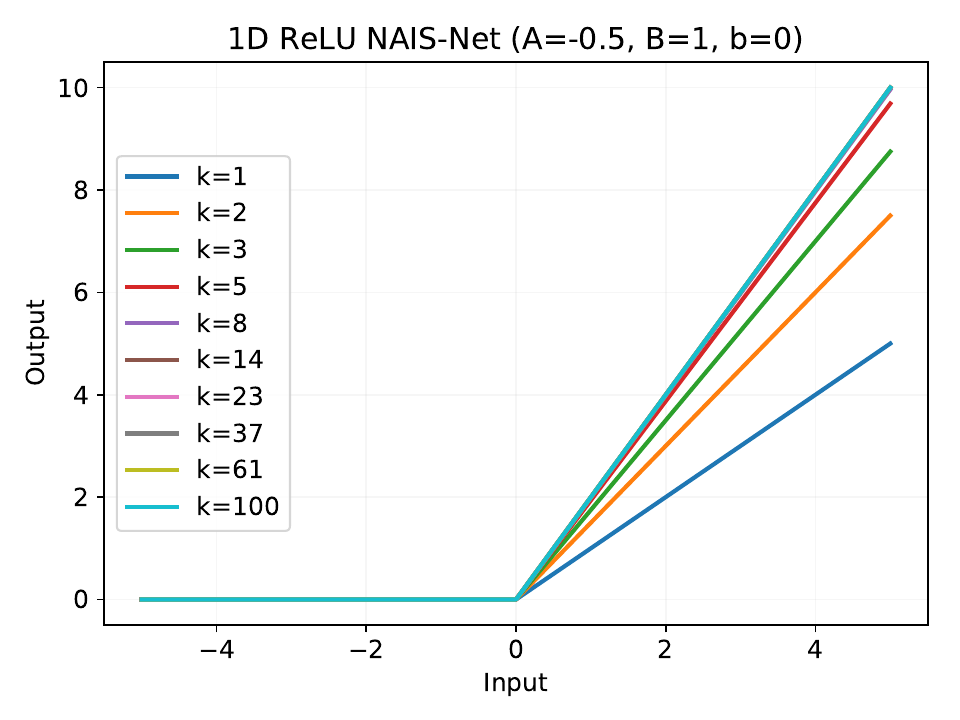}
     \end{figure}
    \end{minipage}
     \begin{minipage}{0.48\columnwidth}
    \begin{figure}[H]
        \centering
            \includegraphics[width=\linewidth, clip]{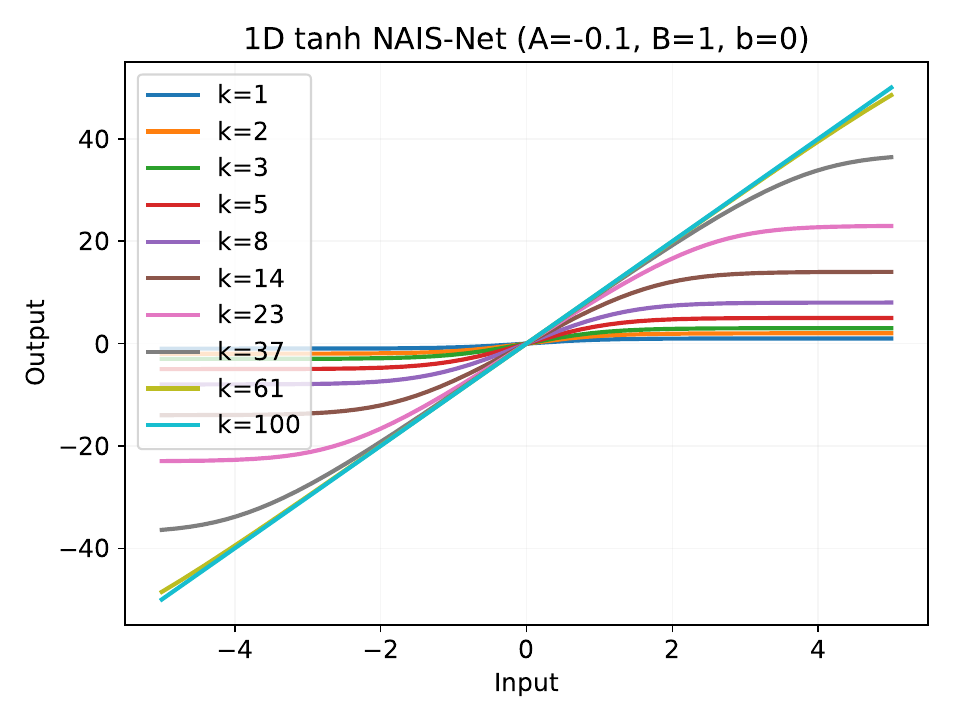}
    \end{figure}
    \end{minipage}
    \noindent
    \begin{minipage}{0.48\columnwidth}
    \begin{figure}[H]
     \includegraphics[width=\linewidth]{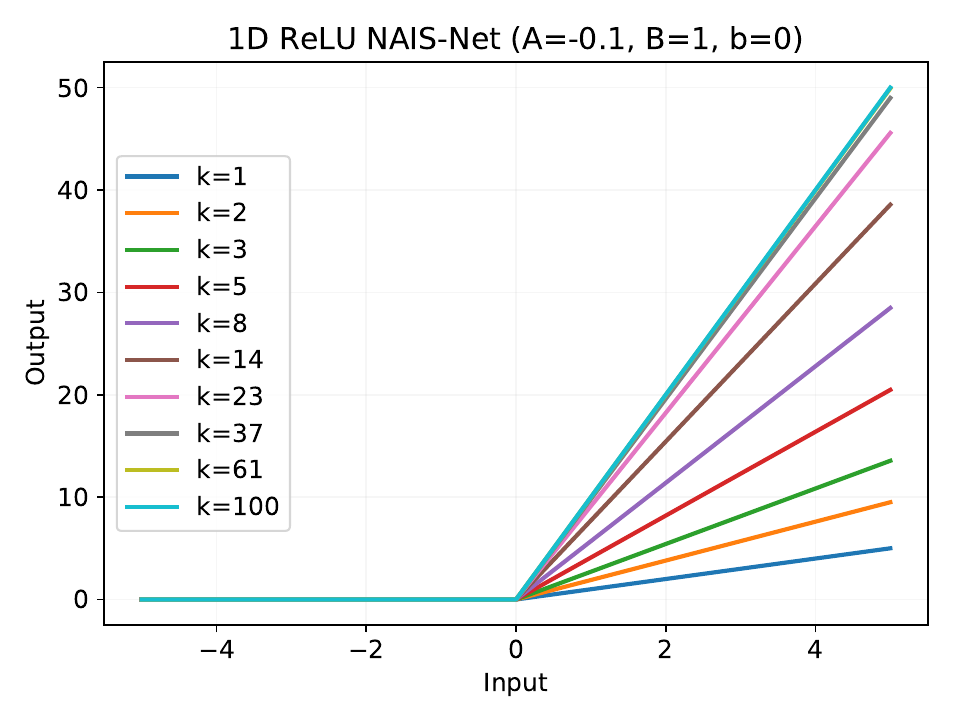}
     \end{figure}
    \end{minipage}

    \caption{
    \small{\bf Single neuron NAIS-Net. Input-output map for different unroll length for tanh (Left) and ReLU (Right) activations.}
    The input-output maps produced by stable NAIS-Net with our proposed reprojection for fully connected architectures. In particular, the top and bottom graphs present the case of positive stable eigenvalues with different magnitude. The left figures shows that the stable tanh activated networks have   monotonic activations (strictly increasing around the origin) that tend to a straight line as $k\rightarrow\infty$ as in our theoretical results. Morever, the map is Lipschitz with Lipschitz constant equal to the steady state gain presented in the main paper, $\|A^{-1}\|\ \|B\|$. Figures on the right present the ReLU case where the maps remain of the changes in slope up to $\|A^{-1}\|\ \|B\|$. This means that one cannot have an unbounded slope or the same map for different unroll length. The rate of change of the map changes as a function of $k$ is also determined by the parameters.
}
    \label{sec:1d_stable_1}
\end{figure}

Figure \ref{sec:1d_stable_2} shows the maps produced by stable NAIS-Net with eigenvalues that are outside the region of our proposed reprojection for fully connected layers but still inside the one for convolutional layers. In particular, the top and bottom graphs present the case of negative real stable eigenvalues with different magnitude.  The left figures shows that as a result the stable tanh activated networks have monotonic activations (not strictly in this case) that still tend to a straight line as expected but present intermediate decaying oscillations as $k\rightarrow\infty$. The map is still Lipschitz but the Lipschitz constant is not always equal to the steady state gain because of the transient oscillations. Figures on the right present the ReLU case where the maps remain identical independently of the parameter $A$. This means that this parameter becomes uninformative and one could argue that in this case gradients would vanish. Note that our reprojection for convolutional layers can theoretically allow for this behaviour while the fully connected version does not. Training with Algorithm 2 has been quite successful on our experiments, however, the above points are of interest and will be further investigated in follow-up work.

\begin{figure}[t!]
\centering
   \begin{minipage}{0.48\columnwidth}
    \begin{figure}[H]
        \centering
            \includegraphics[width=\linewidth, clip]{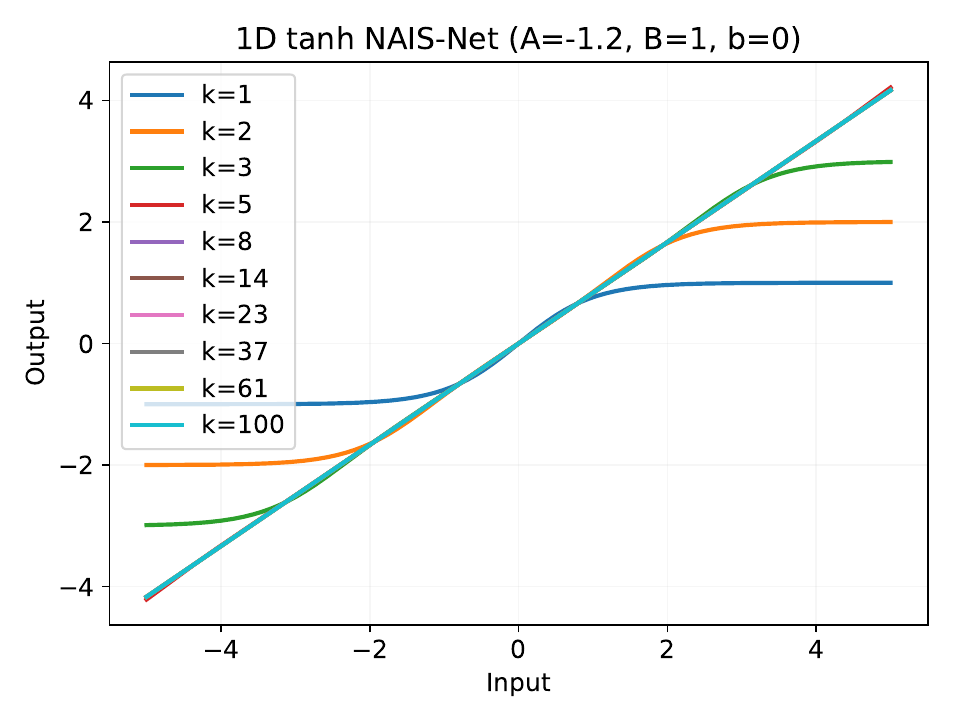}
    \end{figure}
    \end{minipage}
       \noindent
    \begin{minipage}{0.48\columnwidth}
    \begin{figure}[H]
     \includegraphics[width=\linewidth]{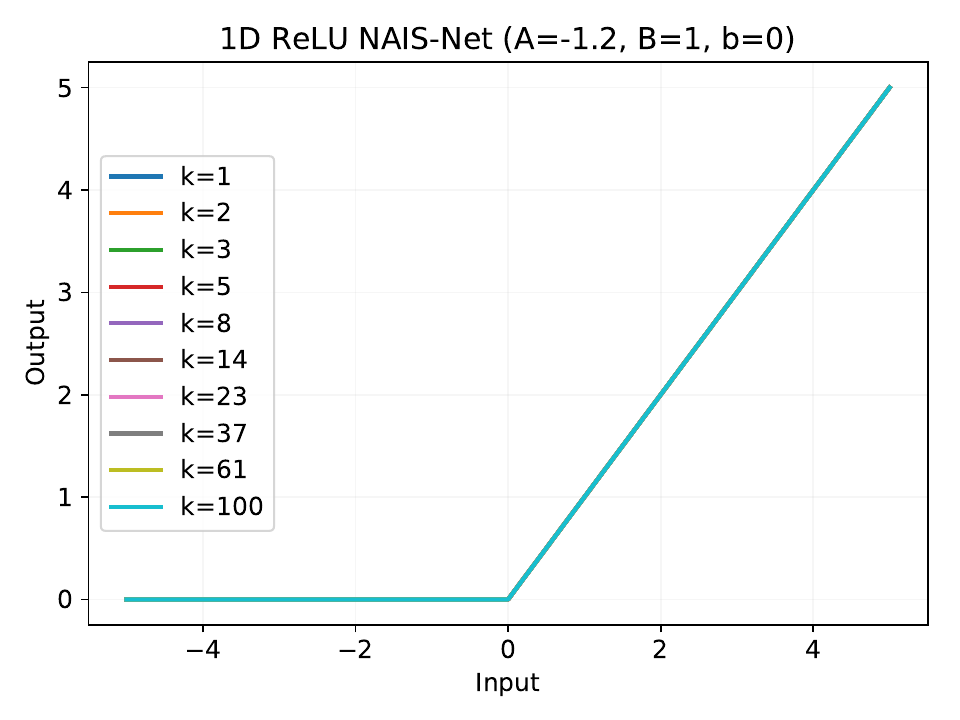}
     \end{figure}
    \end{minipage}
       \begin{minipage}{0.48\columnwidth}
    \begin{figure}[H]
        \centering
            \includegraphics[width=\linewidth, clip]{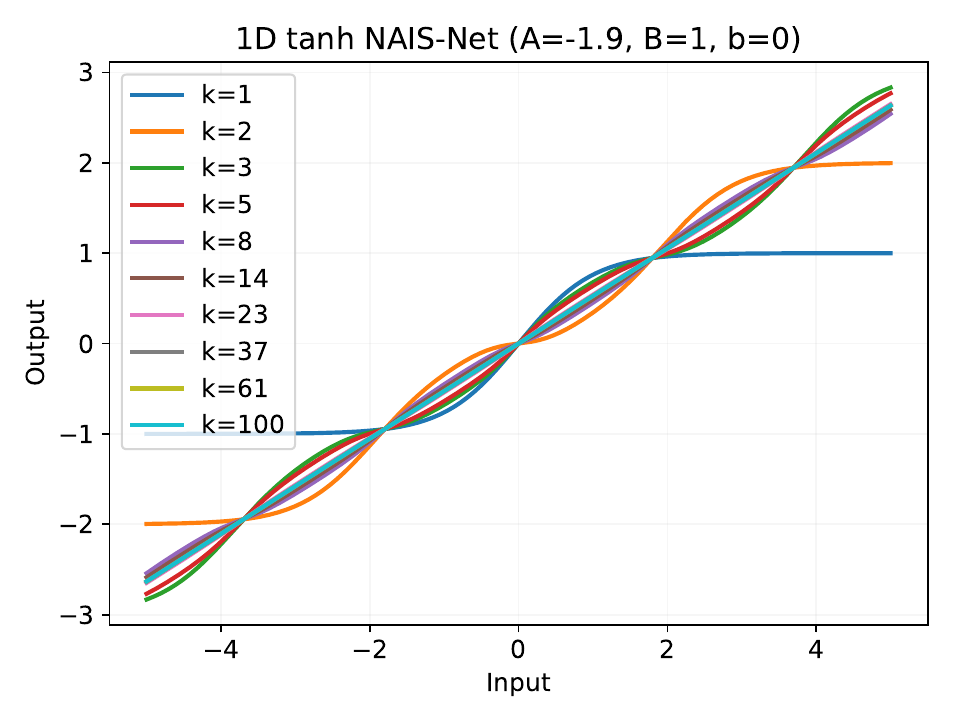}
    \end{figure}
    \end{minipage}
       \noindent
    \begin{minipage}{0.48\columnwidth}
    \begin{figure}[H]
     \includegraphics[width=\linewidth]{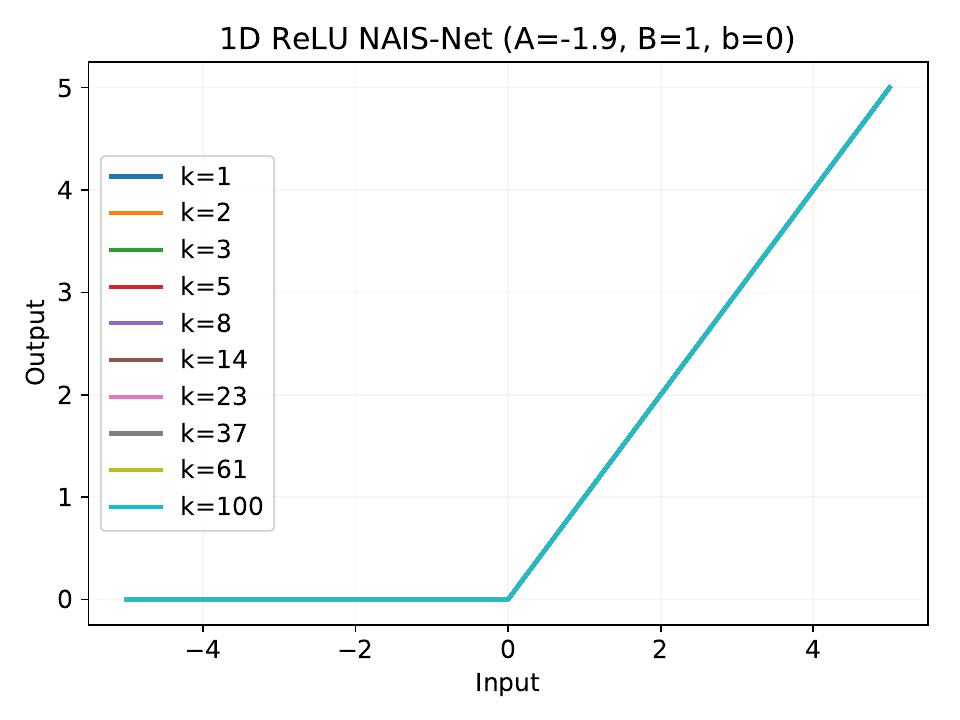}
     \end{figure}
    \end{minipage}

    \caption{
    \small{\bf Single neuron NAIS-Net with less conservative reprojection. Input-output map for different unroll length for tanh (Left) and ReLU (Right) activations.} Input-output maps produced by stable NAIS-Net with eigenvalues that are outside the region of our proposed reprojection for fully connected layers but still inside the one for convolutional layers. In particular, the top and bottom graphs present the case of negative real stable eigenvalues with different magnitude.  The left figures shows that as a result the stable tanh activated networks have monotonic activations (not strictly increasing) that still tend to a straight line as expected but present intermediate decaying oscillations. As a consequence of this the map is still Lipschitz but the Lipschitz constant is not  equal to the steady state gain for any $k$. Figures on the right present the ReLU case where the maps remain identical independently of the parameter $A$.
    }
    \label{sec:1d_stable_2}
\end{figure}

\subsection{Pattern-dependent processing depth through a stopping criterion}
We investigate the use at test time of a stopping criterion for the network unroll, $\|\Delta x(x,u)\|\leq\epsilon$, where $\epsilon$ is a hyper-parameter. The activations for a 1D network are again considered, where $\epsilon$ is set to $0.95$ for illustrative purpose. The resulting activations are discontinuous but locally preserve  the properties illustrated in the previous Section.

Figure \ref{sec:1d_stable_adaptive_computation} shows the input-output maps produced by stable NAIS-Net with our proposed reprojection for fully connected architectures. In particular, the top and bottom graphs present the case of positive stable eigenvalues with different magnitude. The left figures shows that the stable tanh activated networks have piece-wise continuous and locally strictly monotonic activations (strictly increasing around the origin)  that tend to a straight line as $k\rightarrow\infty$ as in our theoretical results. Moreover, the map is also piece-wise Lipschitz with Lipschitz constant less than the steady state gain presented in the main paper, $\|A^{-1}\|\ \|B\|$. Figures on the right present the ReLU case where the maps are piece-wise linear functions with slope that is upper bounded by $\|A^{-1}\|\ \|B\|$. This means that one cannot have an unbounded slope (except for the jumps) or the same map for different unroll length. The rate of change of the map changes as a function of $k$ is also determined by the parameters. The jump magnitude and the slopes are also dependant on the choice of threshold for the stopping criteria.

\begin{figure}[t!]
\centering
    \noindent
    \begin{minipage}{0.48\columnwidth}
    \begin{figure}[H]
        \centering
            \includegraphics[width=\linewidth, clip]{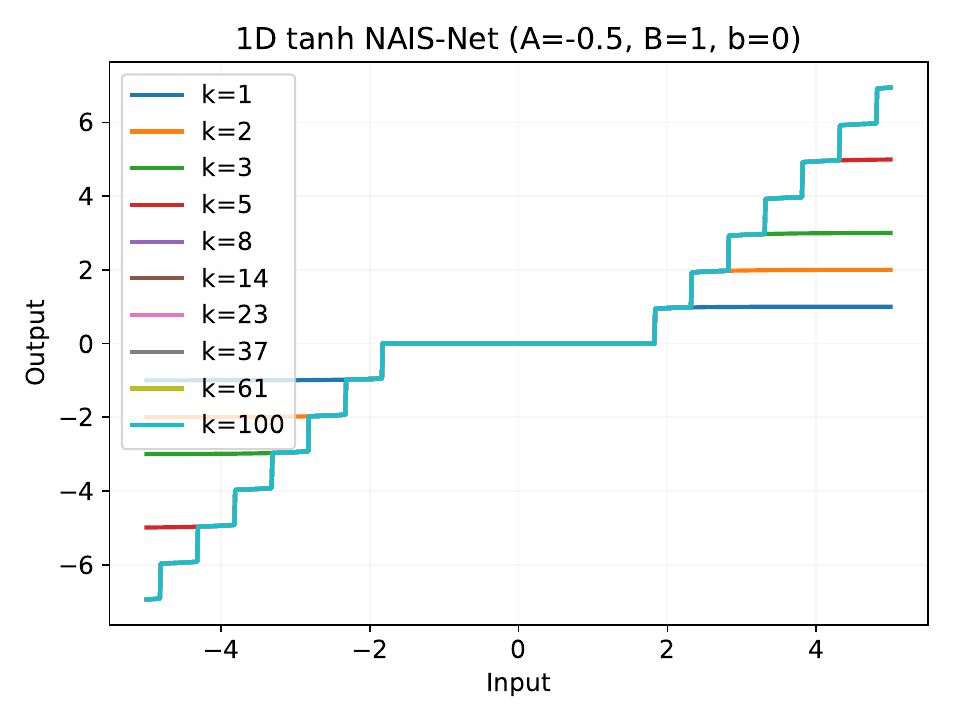}
    \end{figure}
    \end{minipage}
    \noindent
    \begin{minipage}{0.48\columnwidth}
    \begin{figure}[H]
     \includegraphics[width=\linewidth]{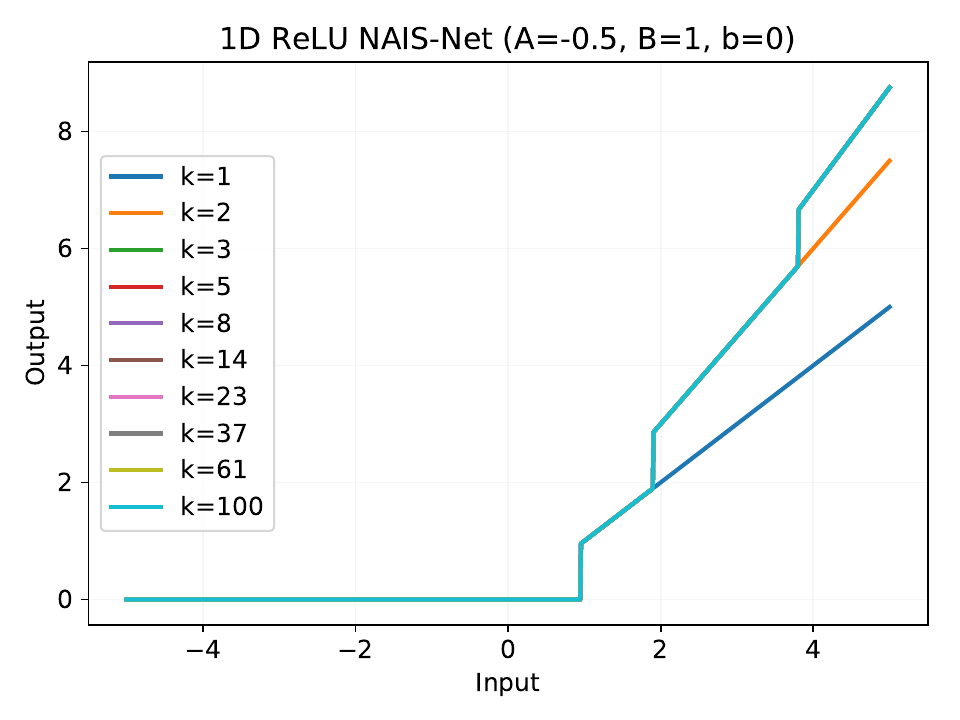}
     \end{figure}
    \end{minipage}
     \begin{minipage}{0.48\columnwidth}
    \begin{figure}[H]
        \centering
            \includegraphics[width=\linewidth, clip]{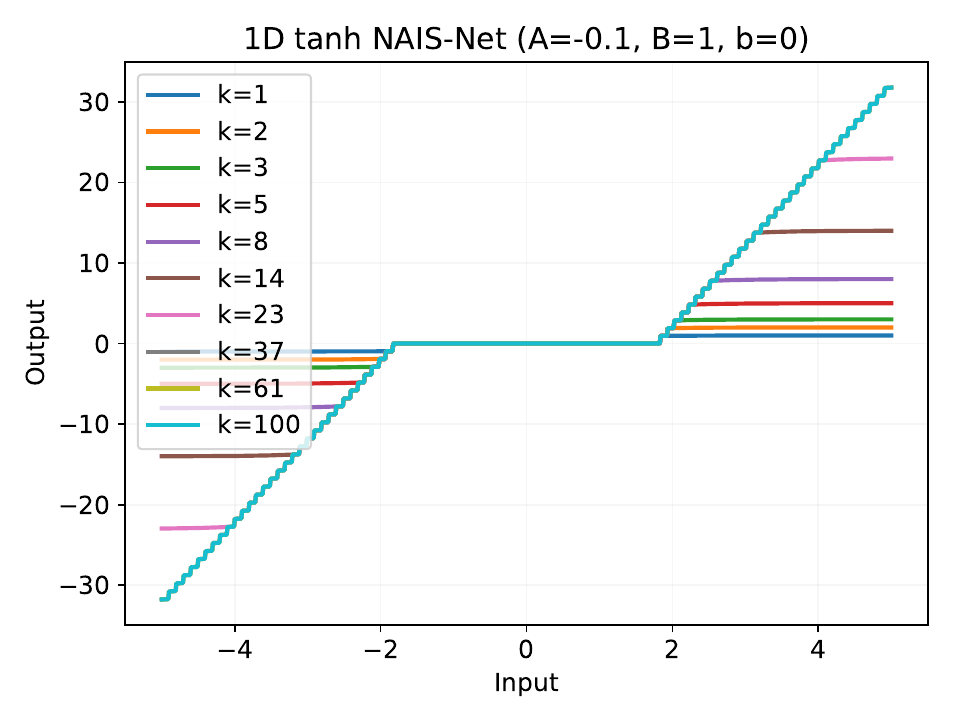}
    \end{figure}
    \end{minipage}
    \noindent
    \begin{minipage}{0.48\columnwidth}
    \begin{figure}[H]
     \includegraphics[width=\linewidth]{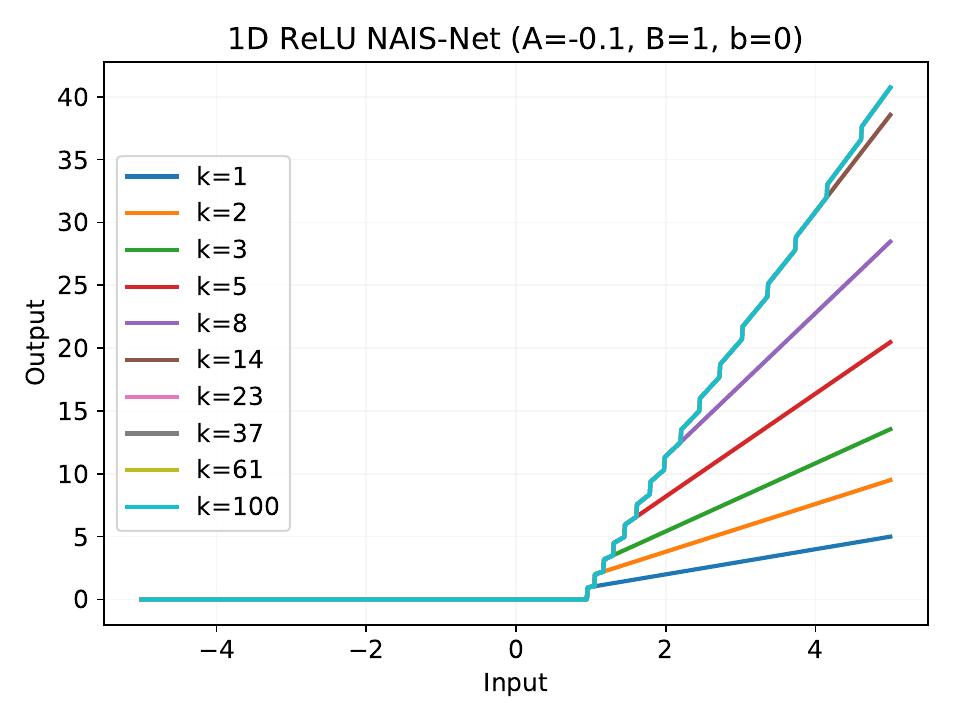}
     \end{figure}
    \end{minipage}

    \caption{
    \small{\bf Single neuron NAIS-Net with variation stopping criteria for pattern-dependent processing depth. Input-output map for different unroll length for tanh (Left) and ReLU (Right) activations.}
    The input-output maps produced by stable NAIS-Net with our proposed reprojection for fully connected architectures. In particular, the top and bottom graphs present the case of positive stable eigenvalues with different magnitude. The left figures shows that the stable tanh activated networks have piece-wise continuous and locally monotonic activations (strictly increasing around the origin)  that tend to a straight line as $k\rightarrow\infty$ as in our theoretical results. Moreover, the map is also piece-wise Lipschitz with Lipschitz constant less than the steady state gain presented in the main paper, $\|A^{-1}\|\ \|B\|$. Figures on the right present the ReLU case where the maps are piece-wise linear functions with slope that is upper bounded by $\|A^{-1}\|\ \|B\|$. This means that one cannot have an unbounded slope (except for the jumps) or the same map for different unroll length. The rate of change of the map changes as a function of $k$ is also determined by the parameters. The jump magnitude and the slopes are also dependant on the choice of threshold for the stopping criteria.
}
    \label{sec:1d_stable_adaptive_computation}
\end{figure}

\begin{sidewaysfigure}

  \centering
  \subfigure[frog]{\includegraphics[width=0.49\textwidth]{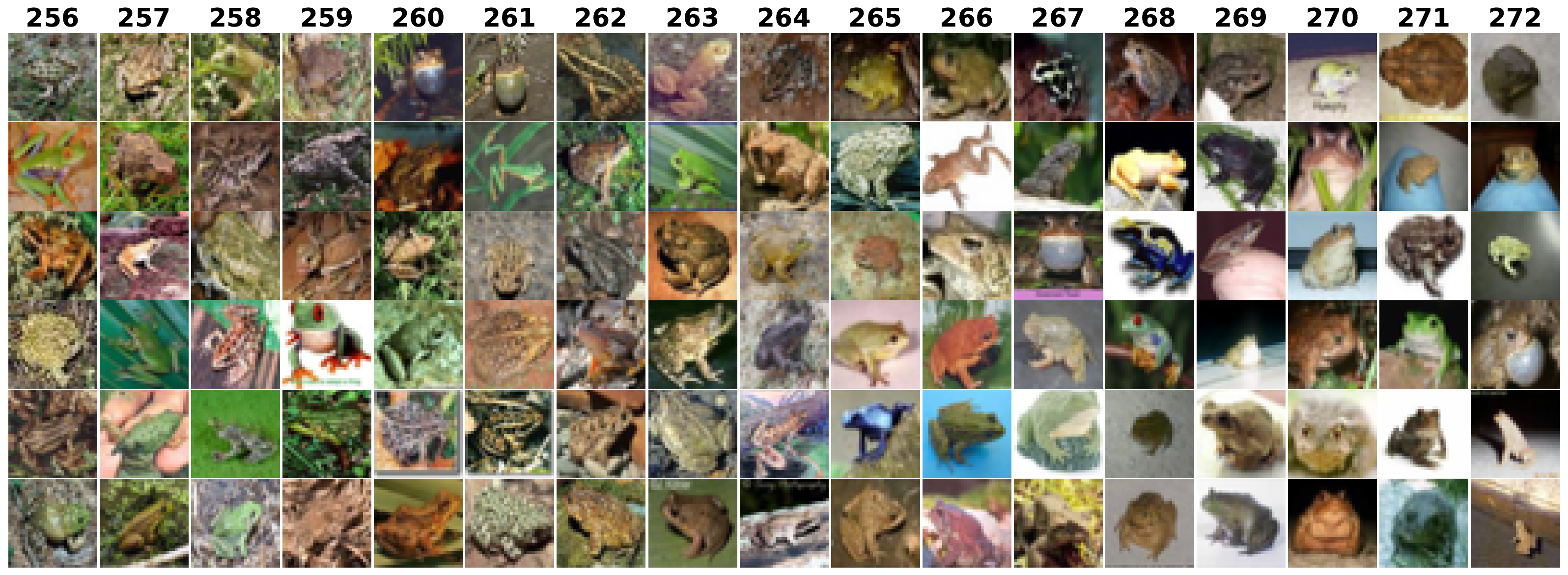}}
    \subfigure[bird]{\includegraphics[width=0.49\textwidth]{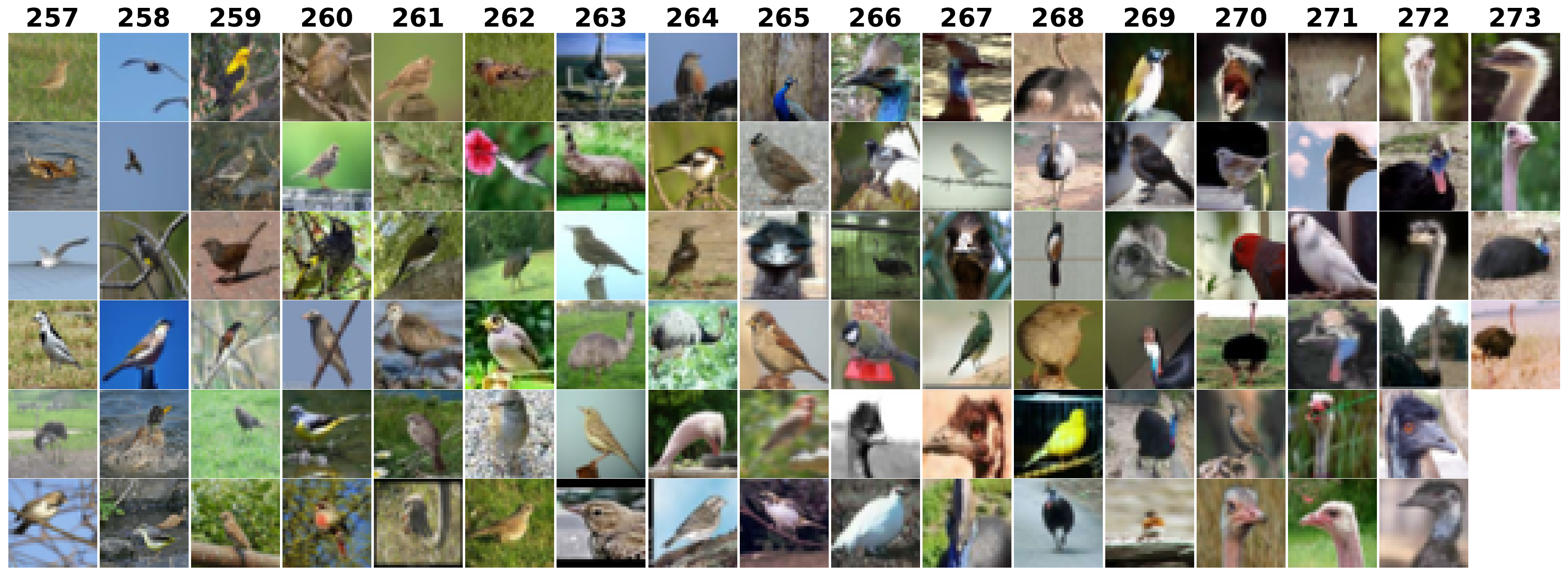}}
    \subfigure[ship]{\includegraphics[width=0.50\textwidth]{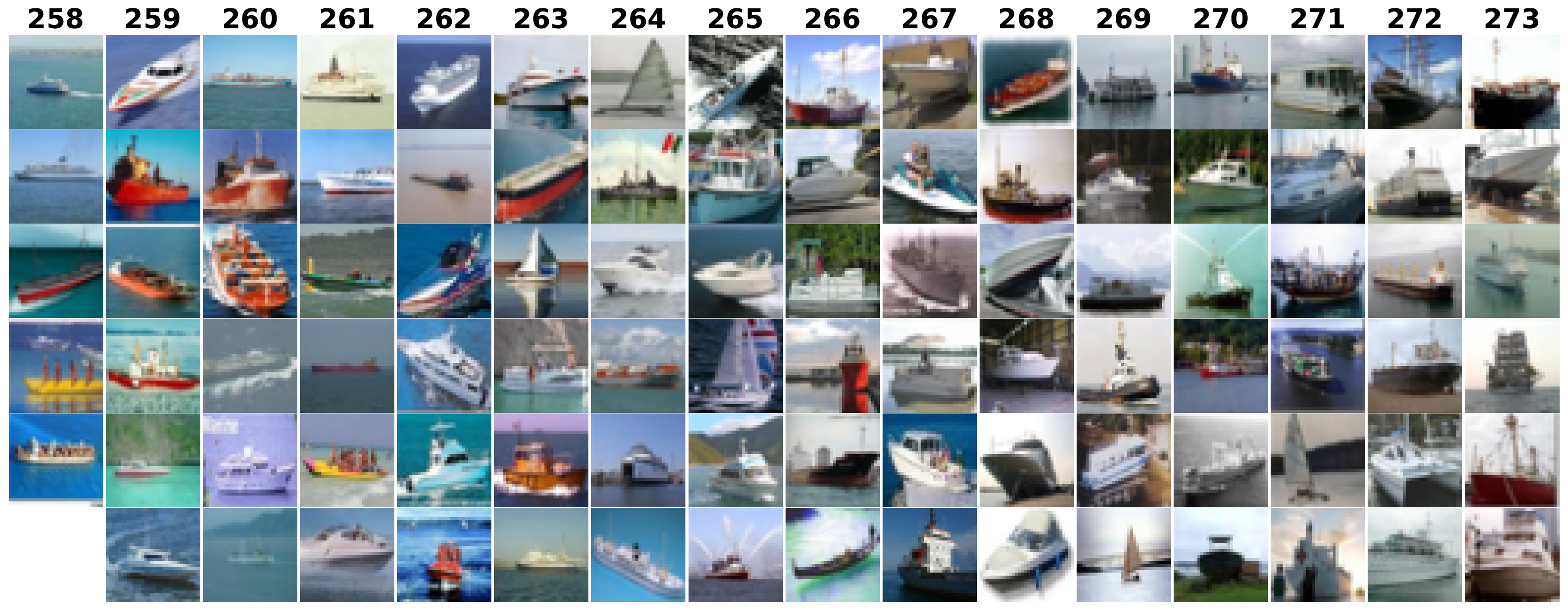}}
    \subfigure[airplane]{\includegraphics[width=0.48\textwidth]{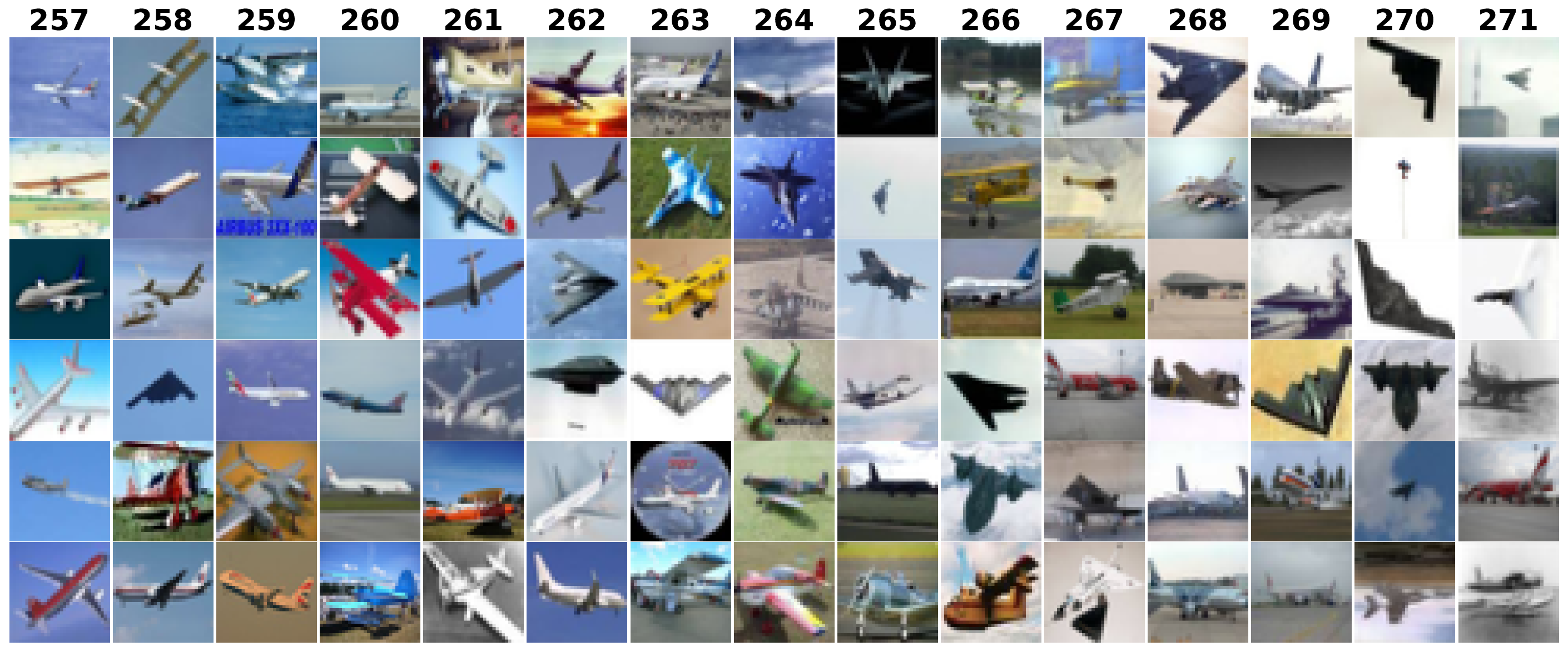}}
    \vspace{-3mm}
    \caption{\small{\bf Image samples with corresponding NAIS-Net depth.} The
      figure shows samples from CIFAR-10 grouped by final network
      depth, for four different classes. The qualitative differences
      evident in images inducing different final depths indicate that
      NAIS-Net adapts processing systematically according
      characteristics of the data.  For example, \emph{``frog''}
      images with textured background are processed with fewer
      iterations than those with plain background. Similarly,
      \emph{``ship''} and \emph{``airplane''} images having a
      predominantly blue color are processed with lower depth than
      those that are grey/white, and \emph{``bird''} images are grouped
      roughly according to bird size with larger species such as
      ostriches and turkeys being classified with greater processing depth.}
    \label{fig:histdepth}
\end{sidewaysfigure}
\medskip

\end{document}